\documentclass[accepted]{uai2026} 
                        
\usepackage[american]{babel}

\usepackage{natbib} 
\usepackage{mathtools} 
\usepackage{booktabs} 
\usepackage{tikz} 

\usepackage{dsfont}
\usepackage{amsmath}
\usepackage{amssymb}
\usepackage{mathtools}
\usepackage{amsthm}
\usepackage{bbm}
\usepackage{enumitem}
\usepackage{comment}
\newtheorem{theorem}{Theorem}[section]
\newtheorem{proposition}[theorem]{Proposition}
\newtheorem{lemma}{Lemma}

\theoremstyle{definition}
\newtheorem{definition}[theorem]{Definition}
\newtheorem{assumption}[theorem]{Assumption}
\theoremstyle{remark}

\newcommand{\E}{\mathbb{E}}
\newcommand{\safe}{\text{ref}}

\usepackage{amsfonts}
\usepackage{hyperref}
\usepackage{adjustbox}
\usepackage{multirow}
\usepackage{algorithm}
\usepackage{algorithmic} 
\usepackage{cleveref}

\hypersetup{
	colorlinks=true,
	linkcolor=black,
	filecolor=blue,
	citecolor = blue,      
	urlcolor=blue,
}

\allowdisplaybreaks

\title{SteinGate: Tail-Sensitive Safe Reinforcement Learning via Stein Discrepancy}

  \author[1]{\href{mailto:yassine.chemingui@wsu.edu>?Subject=Your UAI 2026 paper}{Yassine~Chemingui\thanks{Equal contribution}}{}}
\author[1]{\href{mailto:chenhua.fan@wsu.edu>?Subject=Your UAI 2026 paper}{Chenhua~Fan$^*$}{}}
\author[1]{\href{mailto:honghao.wei@wsu.edu>?Subject=Your UAI 2026 paper}{Honghao~Wei}{}}
\author[1]{\href{mailto:jana.doppa@wsu.edu>?Subject=Your UAI 2026 paper}{Janardhan~Rao~Doppa}{}}
\affil[1]{%
    School of Electrical Engineering and Computer Science\\
    Washington State University\\
    Pullman, Washington, USA
}
  
\begin{document}
\maketitle

\begin{abstract}
Safe reinforcement learning typically enforces safety by bounding expected cumulative costs, a criterion that often fails to detect rare but catastrophic tail events. 
To overcome these limitations, this paper introduces {\em SteinGate}, a boundary-aware distributional safety certificate that replaces fragile tail fitting with a robust consistency check using Kernelized Stein Discrepancy while accounting for boundary atoms induced by clipped costs.
SteinGate evaluates whether observed policy rollout costs remain consistent with a safe reference distribution, providing a non-parametric safety certificate. This certificate is used to dynamically adapt the learning regime: favoring reward-improving policy updates when rollouts remain consistent with the safe reference and switching to recovery behavior when the cost tail deviates. Experiments on continuous-control benchmarks demonstrate that SteinGate significantly reduces both the frequency and severity of constraint violations during training while maintaining competitive returns relative to state-of-the-art baselines.
\end{abstract}

\section{Introduction}

Safe reinforcement learning (RL) aims to deploy autonomous agents in unknown environments where exploration must be restricted due to critical physical or operational constraints. In high-stakes applications such as autonomous robots, medical systems, and industrial control, the inherent trial-and-error nature of an RL agent becomes a significant liability. In these domains, the cost of a single catastrophic failure can outweigh the benefits of thousands of successful runs \citep{zhang2020cautious, yu2021reinforcement}. Consequently, the objective of Safe RL must move beyond simple reward maximization to ensure that safety violations are not merely penalized, but rigorously prevented.

Safe RL problems are formulated as Constrained Markov Decision Processes \citep{altman2021constrained} and safety is typically enforced in an average manner: the expected cumulative cost is required to be controlled below a given threshold \citep{altman2021constrained, achiam2017constrained}. While computationally convenient, these methods suffer from a fundamental ``expectation myopia''. Regulating only the average of the cost distribution allows a policy to satisfy safety limits by balancing rare, catastrophic failures against a high volume of safe episodes. This risk-masking behavior is a critical flaw; a policy may appear safe on paper while remaining dangerously prone to extreme tail events that standard safety mechanisms ignore. This ``expectation myopia'' has motivated a growing interest in risk-sensitive formulations of Safe RL, including chance constraints and coherent risk measures such as Value-at-Risk (VaR) and conditional VaR \citep{sarykalin2008value, chow2018risk}. More recently, extreme value theory has been applied to model the behavior of heavy-tailed distributions directly \citep{gao2026extremevaluepolicyoptimization, ns2024extreme}.

Despite their theoretical appeal, enforcing these constraints in practice is challenging because rare-event estimation is intrinsically unstable in online RL. Policy updates continuously shift the rollout distribution, while trajectory correlations reduce the effective sample size \citep{corrado2026onpolicy}. Moreover, episodic training, termination rules, and cost clipping often induce structural artifacts, such as mass accumulation near zero cost or at violation thresholds, that complicate classical tail modeling assumptions. As a result, direct tail estimation can exhibit high variance and brittleness, leading to oscillations between over-conservatism and catastrophic failure during training.

This paper proposes a novel and fundamentally different perspective for tail-sensitive safe RL: {\em rather than explicitly modeling the tail of the cost distribution, we certify tail risk via distributional comparison}. Our key idea is to upper-bound a tail-sensitive risk functional \citep{rockafellar2000optimization} under the current policy by comparing its rollout cost distribution to a designated safe reference distribution \citep{bellemare2017distributional}. The reference distribution specifies acceptable cost behavior under the desired safety budget and serves as the target of distributional certification. Intuitively, if the current policy’s distribution remains close to a certified safe reference, then its tail risk is provably bounded. This converts safety monitoring into a problem of measuring distributional deviation under non-stationary, sample-based conditions.

To theoretically apply this idea to Safe RL, we leverage Stein’s method, a functional-analytic tool for bounding the difference between expectations under an unknown distribution and a target reference \citep{stein1972bound, gorham2015measuring}. Specifically, we utilize the Kernelized Stein Discrepancy (KSD) \citep{liu2016kernelized} to quantify the divergence between the agent’s current rollout distribution $q_\pi$ for policy $\pi$ and a pre-certified safe baseline. KSD uniquely fits our needs for safe RL because it avoids the need for explicit density estimation. While RL policy rollouts provide an abundance of cost samples, the underlying rollout distribution is generally implicit and intractable. KSD bypasses this bottleneck by comparing empirical samples from $q_\pi$ against the score function ($\nabla \log p$) of a chosen reference distribution $p$.

By combining Stein's method with a tail-sensitive test function, we are able to derive a computable upper bound on the probability of constraint violations. This bound can serve as a rigorous distributional safety certificate, allowing for real-time monitoring of the policy's tail risk. Unlike standard expected costs, this certificate provides the agent with a formal guarantee that its behavior remains aligned with the safe reference, preventing hazardous distributional shifts before catastrophic failure occurs.

We integrate this distributional certificate into the policy optimization loop via a mechanism we call \underline{SteinGate}. Standard Lagrangian methods incorporate safety as a ``soft'' penalty, often leading to training instabilities \citep{stooke2020responsive,spoor2025empirical} or situations in which the reward objective can effectively overpower the constraint when cost penalties are not sufficiently enforced (i.e., constraint trade-offs governed by the Lagrange multiplier) \citep{spoor2025empirical}. In contrast, SteinGate eliminates this risk by using the certified upper bound as a bang-bang controller:

\begin{enumerate}
    \item {\em Reward-Maximization Mode:} Optimization proceeds only when the certificate verifies that the tail risk remains within the prescribed budget.
\item {\em Safety-Recovery Mode:} If the bound is violated, the algorithm immediately switches to prioritize cost reduction, ignoring reward gradients entirely.
\end{enumerate}

Our approach ensures that the agent never sacrifices safety for high rewards; it avoids the instability of traditional methods that rely on tuning safety penalties by balancing the trade-off between reward and cost. We further prove that, under a trust-region condition on the margined policy update, each improvement step preserves the prescribed tail-risk budget with high probability. Our contributions are summarized as follows:

\begin{itemize}
    \item {\em Boundary-Aware Hybrid Stein Certificate:} We derive a tail-risk upper bound for normalized clipped cost distributions in Safe RL. Our framework introduces a hybrid discrete-continuous reference model and split discrepancy that explicitly accounts for boundary atoms induced by cost clipping, enabling rigorous distributional certification without parametric tail fitting.

    \item {\em SteinGate Mechanism and Safety Guarantees:} We integrate this certificate into a lexicographic switching controller that enforces hard safety priority. We provide concentration results for the empirical certificate and establish a trust-region condition under which policy updates preserve the prescribed risk budget with high probability.

    \item {\em Empirical Stability and Performance:} On continuous-control benchmarks, SteinGate significantly reduces violation rates and improves training stability compared to expectation-based and tail-modeling baselines, while maintaining competitive reward performance. Ablations further show stable behavior across cost thresholds, sample sizes, and reasonable reference choices.
\end{itemize}

\section{Problem Setup}
\label{sec:setup}

Safe RL problems are typically formulated as a finite-horizon Constrained Markov Decision Process (CMDP). A CMDP is  defined by the tuple $\mathcal{M}=(\mathcal{S},\mathcal{A},P,r,c,H,\mu_0)$, where $\mathcal{S}$ and $\mathcal{A}$ represent the state and action spaces, $P(\cdot\vert s,a)$ denotes the stochastic transition kernel, $r(s,a)$ is the reward function, $c(s,a)\geq 0$ is the cost function, $H$ is the horizon, and $\mu_0$ is the initial state distribution. For a stochastic policy $\pi(\cdot\mid s)$, we consider the induced trajectory $\tau=(s_0, a_0, \dots, s_{H-1}, a_{H-1}, s_H)$, where $s_0 \sim \mu_0$, $a_t \sim \pi(\cdot\vert s_t)$, and $s_{t+1} \sim P(\cdot\mid s_t, a_t)$. The expected cumulative reward, or reward value function, is defined as:
\begin{align}V_R^\pi(s) = \mathbb{E}_{\pi} \left[ \sum_{t=0}^{H-1} r(s_t, a_t) \vert s_0=s \right].\end{align}
The cumulative episode cost for a specific trajectory $\tau$ is:
\begin{align}
    C(\tau)\;\coloneqq\;\sum_{t=0}^{H-1} c(s_t,a_t).
\end{align}
Given a predefined cost limit $d > 0$, the standard CMDP objective is to maximize the expected cumulative reward subject to an expectation-based safety/cost constraint:
\begin{equation}
\max_{\pi \in \Pi} \mathbb{E}_{s\sim\mu_0} [V_R^\pi(s)] \quad \text{s.t.} \quad \mathbb{E}_{\tau \sim \pi}[ C(\tau) ] \leq d,
\label{eq:cmdp}
\end{equation}
where $\Pi$ denotes the set of feasible policies.

\noindent{\bf Beyond Expected Costs: Tail-Sensitive Constraints.}
In many safety-critical applications, the common constraint on expected cost is insufficient. This is because expectation formulation allows a policy to meet safety limits on average by balancing rare, catastrophic failures against a large number of safe runs. 
To prevent these unacceptable risks, we define the policy risk $C_h^\pi(\mu_0)$ of a policy $\pi$ as the expected value of a smooth surrogate function $h$:
\begin{equation}
C_h^\pi (\mu_0) \coloneqq \mathbb{E}_{\tau\sim\pi} [ h(C(\tau))] = \mathbb{E}_{\pi}\left[ h\left(\sum_{t=0}^{H-1} c(s_t,a_t)\right)\right].
\label{eq:risk}
\end{equation}
We denote $C_h^\pi (\mu_0)$ as $C^\pi_h$ and $\mathbb{E}_{\tau\sim\pi} [ h(C(\tau))]$ as $\mathbb{E}_{C\sim q_\pi} [ h(C)],$ where $q_\pi$ is the distribution of cumulative costs of trajectories induced by a policy $\pi$ to simplify the notation.  In the following sections, for a given policy, we also use $C$ to denote $C(\tau)$, which is a random variable without ambiguity. The surrogate function $h$ here serves as a risk-shaping operator that transforms the cumulative costs into a specific safety metric and provides the flexibility to define different safety objectives. Specifically, $h$ can be implemented as a soft threshold to serve as a smooth proxy for violation probability \citep{cao2020sigmoidal}, or as a barrier penalty (e.g., an exponential or power function) to heavily discourage trajectories that approach high-cost regions \citep{mihatsch2002risk}. It can also be modeled as a survival function to map cost accumulation to a bounded risk of failure \citep{wang2023enforcing}. Ultimately, this flexibility allows the risk functional $C_h^\pi$ to target specific portions of the cost distribution, providing a more precise safety signal than the standard expected cost safety.

Therefore, in this paper, we consider a more general constrained RL objective:
\begin{equation}
\label{eq:risk-dp}
\max_{\pi \in \Pi} \mathbb{E}_{s\sim\mu_0} [V_R^\pi(s)] \quad \text{s.t.}\quad C_h^\pi(\mu_0)\leq \varepsilon.
\end{equation}

\section{Stein-Based Distributional Certification}

When solving the optimization problem in \Cref{eq:risk-dp}, directly estimating the risk function $C_h^\pi$ through Monte Carlo rollouts is often unreliable, particularly in rare-event regimes. When safety violations are infrequent, the empirical estimate of $\mathbb{E}_{C\sim q_\pi} [ h(C)]$ suffers from high variance. This instability is further aggravated by the non-stationary distributional shifts that occur during online policy updating, making the estimate too brittle for reliable safety guarantees. Below, we provide the necessary background on Stein's method and describe how to apply it to address this challenge.

\noindent {\bf Stein's Method and Stein Discrepancies.}
Quantifying the discrepancy between a policy's rollout distribution $q_\pi$ and a reference $p$ is inherently difficult, as both distributions are typically implicit and lack a closed-form density. Stein's method \citep{stein1972bound} provides a principled way to compare a (generally implicit) rollout distribution $q_\pi$ to a reference distribution $p$ by characterizing $p$ through a linear operator. Let $p$ be a target distribution on an open set $\mathcal{C}\subseteq\mathbb{R}^d$ with continuously differentiable density and score function $s_p(x)\coloneqq\nabla_c \log p(x)$. A \emph{Stein operator} $\mathcal{A}_p$ is any linear operator satisfying the Stein identity
 $\mathbb{E}_{C\sim p}[(\mathcal{A}_p f)(C)]=0$ for all functions $f$ in a suitable class where $C$ is a random variable. A canonical choice for absolutely continuous targets is the Langevin--Stein operator \citep{gorham2015measuring}
\begin{equation}
(\mathcal{A}_p f)(c)
\;:=\;
\langle s_p(c), f(c)\rangle
\;+\;
\nabla\!\cdot f(c),
\label{eq:stein-op}
\end{equation}
where $f:\mathbb{R}^d\to\mathbb{R}^d$ is vector-valued and $\nabla\!\cdot f$ denotes the divergence \citep{barbour1990stein}. For any given test function $h$ in an exceedance function class $\mathcal{H}$, the Stein equation seeks a solution $f_h$ such that: 
\begin{equation}
(\mathcal{A}_p f_h)(C) = h(C)-\mathbb{E}_{C\sim p}[h(C)].
\label{eq:stein-eq}
\end{equation}
Thus, we are able to obtain a Stein-type upper bound in the following lemma, which establishes a core {\em safety certificate.} It allows us to bound the unknown risk by measuring how much the current policy drifts from a given baseline {\em without} knowing the rollout distribution $q_\pi$.

\begin{lemma}\label{lem:interior_stein_domination}
    Let $p$ be a target distribution and $q$ be a proposal distribution on $\mathcal{C}$. Let $h: \mathcal{C} \to \mathbb{R}$ be a test function of interest. Suppose there exists a solution $f_h$ to the Stein Equation $\mathcal{A}_p f_h(C) = h(C) - \mathbb{E}_{C \sim p}[h(C)]$ such that $f_h \in \mathcal{F}$ and its norm is bounded by a constant $B(h,p)$, i.e., $\|f_h\|_{\mathcal{F}} \le B(h,p)$. Then, the expectation of $h$ under $q_\pi$ is bounded by:
    \begin{align}
        \mathbb{E}_{C \sim q_\pi}[h(C)] \;\le\; \mathbb{E}_{C \sim p}[h(C)] + B(h,p) \cdot SD(q_\pi\|p),\label{eq:lem-stein-bound}
    \end{align}
    where $SD(q_\pi\|p):=\underset{f_h\in\mathcal{F},\Vert f_h\Vert_\mathcal{F}\leq 1}{\sup}
\left\vert\mathbb{E}_{C\sim q_\pi}\big[(\mathcal{A}_p f_h)(C)\big]\right\vert.$ 
\end{lemma}
We can observe that this certificate is one-sided and monotone: as the policy drifts further from the safe reference (increasing the discrepancy), the upper bound on risk grows accordingly. The proof is in \Cref{ap:proof_of_stein_certi} in the Appendix.

\noindent {\bf Applying Stein Certificate in Safe RL.}
The Stein certificate in \Cref{eq:lem-stein-bound} provides a rigorous approach for applying Stein's method to solve Safe RL problems, specifically for controlling hazardous distributional shifts, a severe issue where an evolving online policy drifts into unsafe cost regimes \citep{achiam2017constrained}. In complex environments, standard monitoring tools are often impractical because the rollout distribution $q_\pi$ induced by the current policy $\pi$ is implicit. While we can generate cost samples via simulation, we lack an explicit formula for its density. Consequently, traditional metrics such as the Kullback-Leibler (KL) divergence become computationally expensive or unstable, as they typically require density ratios or auxiliary density estimation. Stein’s method bypasses these bottlenecks. It only requires the score function ($\nabla \log p$) of the designated safe reference $p$ and empirical samples from the current policy. By leveraging this property, we establish a differentiable and sample-efficient mechanism for constraining policy behavior. This approach ensures that the agent's cost distribution remains aligned with safety requirements without ever requiring the calculation of normalizing constants or explicit probability densities.

\noindent {\bf Practical Instantiation via Kernelized Stein Discrepancy.}
While the Stein certificate is theoretically sound, the supremum in Lemma \ref{lem:interior_stein_domination} is generally intractable. To obtain a computable metric, we restrict $\mathcal{F}$ to the unit ball of a Reproducing Kernel Hilbert Space (RKHS) with base kernel $k(\cdot,\cdot)$. This yields the Kernelized Stein Discrepancy (KSD), defined in its squared population form as:
$
\mathrm{KSD}^2(q_\pi\|p) =
\mathbb{E}_{C,C'\sim q_\pi}\big[u_p(C,C')\big],
$
where $u_p(\cdot,\cdot)$ is the induced Stein kernel determined by $p$ and $k$. In practice, given $m$ rollout samples $\{c_i\} \sim q_\pi$, we employ the consistent V-statistic estimator \citep{chwialkowski2016kernel}:
$\widehat{\mathrm{KSD}}^2_V :=\frac{1}{m^2}\sum_{i=1}^m\sum_{j=1}^m u_p(c_i,c_j)$. By monitoring this discrepancy from a safe reference $p$, we can upper-bound tail-sensitive risk {\em without} requiring parametric modeling or explicit density estimation.

It is useful to distinguish the role of Stein discrepancy from the source of tail sensitivity in our framework. The Stein discrepancy is a general distributional discrepancy: it measures how far the rollout-induced cost distribution deviates from a designated reference distribution through samples from ($q_\pi$) and the score function of ($p$). Its main computational advantage in online RL is that it avoids both density estimation for the implicit rollout distribution and parametric fitting of rare exceedance events, which are unstable under non-stationary policy updates. Tail sensitivity enters through the certificate in Lemma 1 by choosing a risk-shaping test function (h) that emphasizes high-cost events and by comparing rollouts to a reference distribution that encodes admissible safe behavior. Thus, KSD provides a tractable mechanism for detecting distributional drift, while the safety semantics of the certificate are determined by the choice of (h), the reference distribution, and the boundary-aware treatment of clipped costs described in the next section.

\section{SteinGate Framework}

While the Stein certificate provides a principled upper bound on tail risk, its application to RL must account for the compact support of cost distributions. In Safe RL problems, cumulative costs are typically bounded, often exhibiting significant probability mass at the boundaries, such as clusters of zero-cost ``perfect'' episodes. More importantly, safety specifications are almost always defined relative to a fixed threshold $d$. This creates a saturation of risk: once the budget is exceeded, the exact magnitude of the cost becomes secondary to the fact of the violation itself.

To align our mathematical framework with this reality, we utilize a threshold-normalized cost representation. By clipping and scaling costs to the interval $[0, 1]$, we ensure that all budget exceedances are treated as uniformly catastrophic events. This mapping ensures that the distributional certificate remains practically relevant by focusing the ``sensitivity'' of the Stein operator on the region where the budget is consumed, while maintaining the mathematical consistency required for a bounded function class.

\subsection{Boundary Atoms}

To encode budget exceedances, we define the normalized clipped cumulative cost $ \tilde{C}(\tau) = \min\{C(\tau)/d,\,1\}$, then $\tilde{C}(\tau)\in[0,1]$ normalized and the boundary value $\tilde{C}(\tau)=1$ aggregates all budget violations into a single event. While this representation is ideal for defining a stable reference on $[0,1]$, it creates a structural mismatch with standard Stein machinery. Specifically, the rollout distribution of $\tilde{C}(\tau)$ is often a {\em mixed} distribution. It frequently contains discrete probability mass (atoms) at the boundaries, a ``safe'' atom at $\tilde{C}(\tau)=0$ and a ``violation'' atom at $\tilde{C}(\tau)=1$ with a continuous component on the interior $(0,1)$ as shown in Figure \ref{fig:boundary_atoms}. This structure conflicts with the standard Langevin-Stein operator, which assumes a smooth density on an open domain.
\begin{figure}
    \centering
    \includegraphics[width=\linewidth, height=3.8cm, ]{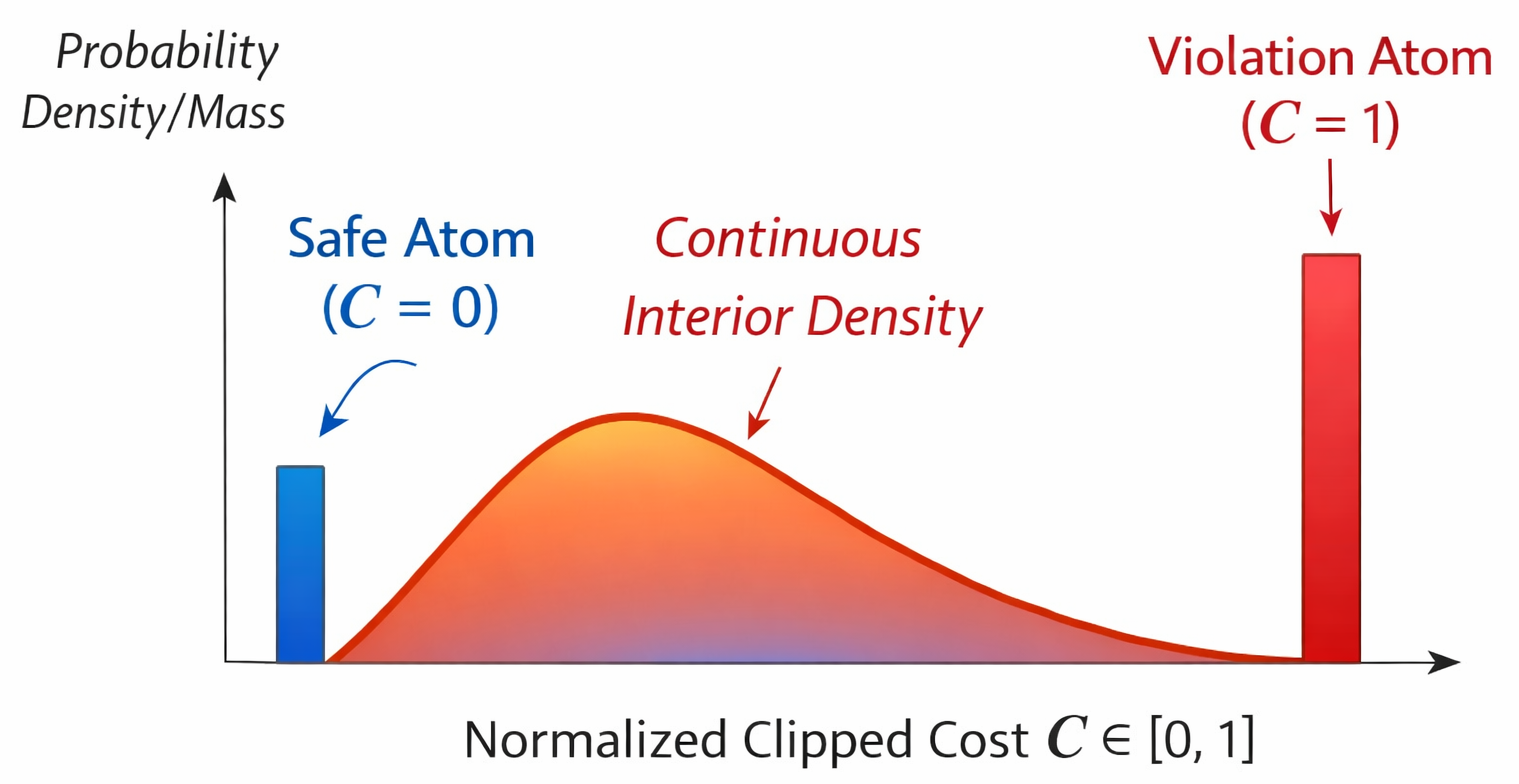}
    \caption{Hybrid cost distribution over normalized clipped cumulative cost $C \in [0,1]$: a safe atom at $C=0$, a violation atom at $C=1$, and a continuous interior density on $(0,1)$.}
    \label{fig:boundary_atoms}
    
\end{figure}
The threshold-normalized setting presents two challenges. {\bf 1) Boundary Atoms:} Stein identities rely on integration by parts, which requires boundary terms to vanish. However, clipping induces discrete probability mass (atoms) at $\{0, 1\}$, which standard continuous operators fail to capture with spikes at the boundaries. {\bf 2) Endpoint Singularity:} Score functions for common interior references (e.g., Beta distributions) often become singular near the boundaries. This can amplify variance and cause numerical instability in KSD estimators. To resolve these issues, we propose a hybrid Stein operator that explicitly accounts for boundary mass while maintaining the stability of the interior reference.

\subsection{Hybrid Reference and Discrepancy}
\label{subsec:hybrid_ref}
\paragraph{Safe reference.}
To align with the mixed structure of threshold-normalized and clipped costs, we model admissible behavior using a hybrid reference distribution that explicitly accounts for boundary atoms:
\begin{equation}
    p_{\text{ref}}
    =
    a_0^\star\,\delta_{0}
    +
    a_1^\star\,\delta_{1}
    +
    \bigl(1-a_0^\star-a_1^\star\bigr)\,p_{\mathrm{int}},
    \label{eq:psafe}
\end{equation}
where $\delta_0,\delta_1$ denote point masses at the boundaries, $a_0^\star,a_1^\star\in[0,1]$ are the reference boundary masses, and $p_{\mathrm{int}}$ is a smooth interior reference density on $(0,1)$ (e.g., a Beta distribution).
The reference distribution $p_{\mathrm{ref}}$ should be interpreted as a safety specification rather than as an estimate of the unknown cost distribution induced by an optimal safe policy. The optimal safe rollout distribution depends on the environment dynamics and the learned policy, and is generally not available a priori. The role of $p_{\mathrm{ref}}$ is instead to encode admissible behavior in the normalized cost space. The boundary masses $a_0^\star$ and $a_1^\star$ specify acceptable mass at zero cost and at the violation boundary, while the interior density $p_{\mathrm{int}}$ specifies the desired shape of costs that remain below the threshold.

For any distribution $q$ on $[0,1]$, we define the empirical boundary masses and the total interior weight as:
$$
a_0(q) \coloneqq \Pr(C=0\vert C\sim q),
a_1(q) \coloneqq \Pr(C=1\vert C\sim q), 
$$
Let $q_{\mathrm{int}}$ denote the conditional distribution of $q$ on the interior $(0,1)$. Furthermore, let $w(q) \coloneqq 1 - a_0(q) - a_1(q).$ This decomposition allows us to analyze the safety of a policy by separately evaluating its performance at the boundaries and its behavior within the threshold.
The reference need not be exactly realizable by every MDP. If the specification is overly restrictive or structurally misaligned with the achievable cost distributions, this appears as a persistent discrepancy between $q$ and $p_{\mathrm{ref}}$. From the perspective of the certificate, such mismatch is indistinguishable from unsafe distributional behavior: both indicate that the observed rollout distribution does not satisfy the prescribed safety specification. This is intentional, since the certificate is defined relative to the user-provided reference rather than to an unknown optimal distribution. In practice, the choice of $p_{\mathrm{ref}}$ controls conservatism: stricter references induce larger discrepancies for the same rollout distribution, while more relaxed references admit a wider range of interior cost.

\paragraph{Discrepancy with split penalties.}
To address the mixed structure of cost distributions, we define a hybrid discrepancy that separately manages boundary mass shifts and interior distribution mismatches:
\begin{equation}
\begin{aligned}
&\mathrm{SD}_{\mathrm{tot}}(q\|p_{\mathrm{ref}})
\coloneqq \lambda_{\mathrm{disc}}\,
\mathrm{SD}_{\mathrm{disc}}(q\|p_{\mathrm{ref}}) \\
&\quad\quad\quad\quad\quad+ \lambda_{\mathrm{int}}\, w(q)\,
\mathrm{SD}_{\mathrm{int}}(q_{\mathrm{int}}\|p_{\mathrm{int}}) \\
\le &\
\lambda \ \Big(
\mathrm{SD}_{\mathrm{disc}}(q\|p_{\mathrm{ref}}) 
+ w(q)\cdot\mathrm{SD}_{\mathrm{int}}(q_{\mathrm{int}}\|p_{\mathrm{int}})
\Big),
\end{aligned}
\label{eq:sd-tot}
\end{equation}
where $\lambda\coloneqq\max\{\lambda_{\mathrm{disc}},\lambda_{\mathrm{int}}\} >0 .$ The discrete term $\mathrm{SD}_{\mathrm{disc}}(q\|p_{\mathrm{ref}})$ can be written as $\mathrm{SD}_{\mathrm{disc}}(q\|p_{\mathrm{ref}})
    :=\bigl(a_0^\star-a_0(q)\bigr)_+
    +
    \bigl(a_1(q)-a_1^\star\bigr)_+,$
where $(x)_+=\max(x,0)$. This term is one-sided: it penalizes excess violation mass ($a_1(q)>a_1^\star$) and loss of safe mass ($a_0(q)<a_0^\star$), while shifts in the harmless direction, fewer violations or more zero-cost episodes than the reference, incur no penalty. The interior term $\mathrm{SD}_{\mathrm{int}}$ is a standard Stein discrepancy relative to $p_{\mathrm{int}}$ and can be computed on the interior conditional $q_{\mathrm{int}}$ using an appropriate Stein operator for $p_{\mathrm{int}}$ (e.g., the Langevin--Stein operator on a clipped interior).

Together with the Stein certificate of Lemma~\ref{lem:interior_stein_domination}, the hybrid discrepancy in~\eqref{eq:sd-tot} yields a certificate that is sensitive to safety-critical violation mass at the boundary while still controlling continuous interior mismatch.

\subsection {SteinGate: A Bang-bang Controller}

We are now ready to introduce our empirical hybrid certificate to govern policy optimization. For a given reference $p_{\mathrm{ref}}$ and a parametrized policy $\pi_\theta$, we define the empirical Stein certificate as:
\begin{equation}
U(\theta)
=
\mathbb{E}_{C\sim p_{\mathrm{ref}}}[h(C)]
+
\lambda\, \widehat{\mathrm{SD}}_{\mathrm{tot}}(q_{\pi_\theta} \| p_{\mathrm{ref}}),
\label{eq:hybrid-certif}
\end{equation}
where the total discrepancy is estimated from the most recent rollout samples.

Instead of treating safety as a soft penalty, we implement this certificate as a bang-bang controller compared against a safety threshold $\varepsilon$. The agent's behavior is determined by a strict switching rule:
\begin{equation}
\text{Mode}(\theta)
=
\begin{cases}
\text{reward-maximization}, & \text{if } U(\theta) \le \varepsilon, \\
\text{safety-recovery}, & \text{if } U(\theta) > \varepsilon.
\end{cases}
\end{equation}

The scalar $\lambda$ in $U(\theta)$ does not play the same role as a Lagrange multiplier in constrained policy optimization. In Lagrangian methods, the multiplier weights the cost term inside the optimization objective, so reward and safety gradients are combined in a single update direction. In SteinGate, by contrast, $\lambda$ only scales the distributional certificate used by the gate. Once the mode is selected, the policy update is either reward-driven or recovery-driven; the two gradients are not linearly combined. Thus, $\lambda$ controls the conservativeness of the safety test, and hence the point at which switching occurs, rather than mediating a continuous trade-off between reward maximization and constraint satisfaction.

The same gating interpretation also determines how reference misspecification is handled algorithmically. If the certificate remains above the threshold across successive iterations, this may indicate either genuinely unsafe rollout behavior or a reference specification that is too restrictive for the environment. The controller treats both cases conservatively, since both correspond to failure to satisfy the prescribed distributional safety certificate. Persistent selection of the recovery mode can therefore serve as a practical diagnostic for reference misspecification, suggesting that boundary masses or the interior reference require relaxation.

This framework, summarized in Algorithm~\ref{alg:steingate}, is what we refer to as ``SteinGate.'' This design provides two critical technical advantages over standard Lagrangian penalty approaches. First, Lagrangian approaches optimize a linear combination of objectives. 
In this setting, the relative weighting between reward and cost is governed by a learned multiplier, and poorly tuned or unstable updates can cause the agent to prioritize reward performance while violating safety constraints, since the cost term is only enforced in expectation.
SteinGate establishes strict lexicographic priority: when the rollout distribution drifts into the unsafe tail region, the reward gradient is ignored entirely, forcing the agent to prioritize recovery. Second, it circumvents the well-known instability of dual-variable tuning, replacing oscillatory dual ascent updates with a robust, boundary-aware distributional check. 
This design effectively stabilizes the learning process, as evidenced by the empirical evaluations in Section~\ref{subsec:results}.

Although the controller uses a discontinuous switching rule, SteinGate does not switch on instantaneous costs. The gate is evaluated using a batch-level distributional certificate computed from recent rollouts, so isolated high-cost trajectories are averaged through the empirical discrepancy rather than immediately changing the update mode. Moreover, the trust-region update used by the backbone optimizer limits the per-iteration change in the policy, which in turn limits abrupt changes in the induced cost distribution. The two modes are also asymmetric: reward updates are allowed only while the certificate remains below the prescribed budget, whereas recovery updates explicitly act to reduce the cost violation once the budget is exceeded. Together, these features restrict switching signal evolution to the distributional scale of policy updates rather than individual trajectory noise.

\begin{algorithm}[t]
\caption{The SteinGate Framework}
\label{alg:steingate}
\begin{algorithmic}[1]
\REQUIRE Policy $\pi_{\theta}$; backbone optimizer $\textsc{OptimizerStep}(\cdot)$; cost samples $\{c_i\}_{i=1}^{n} \sim \pi_{\theta}$; risk function $h(\cdot)$; safe reference $p_{\text{ref}}$; risk budget $\varepsilon$; weight $\lambda \ge 0$
\ENSURE Updated policy $\theta_{new}$

\item[] \textbf{Phase 1: Certification (The Gate)}
\STATE \textit{Reference risk:} $R_p \leftarrow \mathbb{E}_{p_{\text{ref}}}[h(C)]$
\STATE \textit{Estimate discrepancy:} $D \leftarrow \widehat{\mathrm{SD}}(\{c_i\} \mid p_{\text{ref}})$
\STATE \textit{Compute certificate:} $U \leftarrow R_p + \lambda \, D$
\STATE $\textsc{IsSafe} \leftarrow (U \le \varepsilon)$

\item[] \textbf{Phase 2: Policy Optimization (Control)}
\IF{\textsc{IsSafe}}
    \STATE Set $\mathsf{mode}\leftarrow$ \textsc{RewardOpt} 
    \hfill \COMMENT{\textcolor{teal}{Certificate holds}}
\ELSE
    \STATE Set $\mathsf{mode}\leftarrow$ \textsc{CostOpt} 
    \hfill \COMMENT{\textcolor{purple}{Safety recovery}}
\ENDIF
\STATE $\theta \leftarrow \textsc{OptimizerStep}(\theta;\mathsf{mode})$

\end{algorithmic}
\end{algorithm}

\subsection{Safety Guarantee for SteinGate}
Before presenting our main results, we first make the following assumption on the empirical certificate error.
\begin{assumption}[Bounded Stein kernel on the clipped interior]
\label{ass:bounded-kernel-1}
After interior clipping, the induced Stein kernel $u_{p_{\mathrm{int}}}(x,y)$
used by the KSD estimator satisfies $|u_{p_{\mathrm{int}}}(x,y)|\le M_u$ for all clipped $x,y$.
\end{assumption}

\begin{proposition}[Concentration of the empirical certificate]
\label{prop:certificate-concentration}
Assume Assumption~\ref{ass:bounded-kernel-1} holds. For any given parametrized policy $\pi_\theta$ and a reference distribution $p_{\mathrm{ref}},$ denote the Stein certificate as: 
$
U_{\textrm{stein}}(\theta)=\mathbb{E}_{p_{\mathrm{ref}}}[h(C)]
+
\lambda\, {\mathrm{SD}}(q_\theta \| p_{\mathrm{ref}}),
$ and the empirical Stein certificate is defined in \Cref{eq:hybrid-certif}.
Given $n$ rollout samples $\{c_i\}\sim q_\theta$, and $m$ interior samples,  
for any $\delta\in(0,1)$, with probability at least $1-\delta$,
\begin{equation*}
\begin{aligned}
\big|{U}(\theta)
- U_{\textrm{stein}}(\theta)\big|
\le \mathcal{O}\!\left( \sqrt{\frac{1}{m}} + \left(1+\mathrm{SD}_{\mathrm{int}}\right)\sqrt{\frac{1}{n}} \right)
\end{aligned}
\end{equation*}
\end{proposition}

We are now ready to present the main results. The following theorem provides a sufficient condition for ensuring a safe policy after each policy improvement step. The detailed proof can be found in Appendix \ref{ap:stein-cpo}.


\begin{theorem}[SteinGate Safety Bound]
\label{thm:stein-cpo}
Let $\pi_{k}$ be the current policy and let $\pi_{{k+1}}$ satisfy the trust-region constraint
\eqref{eq:stein-tr} in Appendix \ref{ap:stein-cpo}. If the Stein factor at $k^{\text{th}}$ iteration $\lambda_k$ satisfies
\[
\lambda_k \ge
\frac{
2H\,\epsilon_h^{\pi_{k+1}}
\displaystyle
\sum_{t=0}^{H-1}
\E_{\bar{s}\sim \bar{d}^{\pi}_k}
\!\left[
D_{\mathrm{TV}}\big(
\pi_{k+1}(\cdot\mid s),
\pi_k(\cdot\mid s)
\big)
\right]
}{
\underline{SD}_{\rm tot}^+(q_{\pi_k}\|p_{\mathrm{ref}})
}.
\]
where $\underline{SD}_{\rm tot}^+(q_{\pi_k}\|p_{\mathrm{ref}})$ takes the form in \eqref{eq:sd_lcb_def}, 
$\epsilon_h^{\pi_{k+1}}$ takes the form in \eqref{eq:tv_err}
.
Then with high probability at least $1-2\delta$, the tail-risk for any given test function $h\in\mathcal{H}$ satisfies
\begin{equation}
\begin{aligned}
\label{eq:stein-bound-main}
&C_h^{\pi_{k+1}}(\mu_0)
\le
\varepsilon
.
\end{aligned}
\end{equation}
\end{theorem}

Theorem 4.3 gives a sufficient condition under which a policy update preserves the prescribed risk budget. The lower bound on $\lambda_k$ specifies how conservative the certificate must be relative to the size of the policy change, as measured by the trust-region distance. Since the certificate is one-sided, increasing $\lambda_k$ tightens the gate by assigning greater weight to distributional deviation from the safe reference. This differs from a Lagrangian trade-off parameter: $\lambda_k$ does not determine the relative magnitude of reward and cost gradients in a shared objective, but only the conservativeness of the safety test used before selecting the update mode. Under the stated trust-region condition, any value above the bound is sufficient to keep the post-update risk within the budget with high probability.

\section{Experiments and Results}
\label{sec:experiments}
This section empirically evaluates SteinGate on a suite of continuous-control safety tasks. Our experiments are designed to examine whether SteinGate can alleviate the trade-off between task performance and safety violations that arises in constrained RL.
We further study the robustness of SteinGate under varying cost thresholds for the safety constraint and analyze its sensitivity to the choice of the interior reference distribution used in the safety updates.

\subsection{Experimental Setup}
\noindent \textbf{Benchmarks.}
We employ the {\em Safety Gymnasium} and {\em Safety MuJoCo} suites \citep{ ray2019benchmarking, zhang2020first}. Safety Gymnasium tasks (PointGoal, PointCircle, CarGoal, CarCircle) simulate autonomous navigation where agents must reach targets while avoiding hazards or constraining movement patterns. Safety MuJoCo tasks (Ant, HalfCheetah, Humanoid, Swimmer) focus on locomotion control with velocity limits, testing the agent's ability to maintain stability under state-dependent constraints.

\noindent \textbf{Baselines.}
We compare SteinGate against multiple safe RL methods. 1)
\textsc{CPO}  \citep{achiam2017constrained} represents the standard for expectation-based constraints, enforcing safety on the average cost. Notably, we utilize CPO as the base optimization engine for SteinGate.
2) \textsc{EVO}  \citep{gao2026extremevaluepolicyoptimization} utilizes extreme value theory to model the cost distribution tail, explicitly targeting rare, high-impact violations.
3) \textsc{Saute} \citep{sootla2022saute} and 4) \textsc{Simmer} \citep{sootla2022effects} employ state augmentation to strictly enforce constraints, with Simmer additionally inducing a curriculum on the safety budget.

\noindent \textbf{Evaluation Protocol.} All methods are trained with identical interaction budgets (10M steps) and environmental configurations. We report performance over six random seeds, visualizing the mean and standard deviation of episodic return and cost. We introduce two quantitative metrics to assess convergence: 1) \textit{Tail Feasibility} measures the fraction of episodes satisfying the cost constraint during the final 20\% of training, and 2) \textit{Feasible Return} reports the average episodic return conditioned on constraint satisfaction, isolating the performance of valid policies.

\noindent \textbf{Implementation Details.} To guarantee a fair comparison, all algorithms utilize identical policy and value network architectures and share core hyperparameters; this ensures that observed performance differences stem solely from the safety mechanisms. 
The practical estimation procedure of Stein certificate $U(\theta)$ is detailed in Appendix~\ref{sec:practical}.
We employ the \texttt{OmniSafe} implementations for standard baselines and the official public codebase for EVO. Detailed hyperparameter configurations are provided in the Appendix \ref{ap:hyperparameters}.

\subsection{Results and Discussion}
\label{subsec:results}
\begin{figure*}
    \centering
    \includegraphics[width=\linewidth]{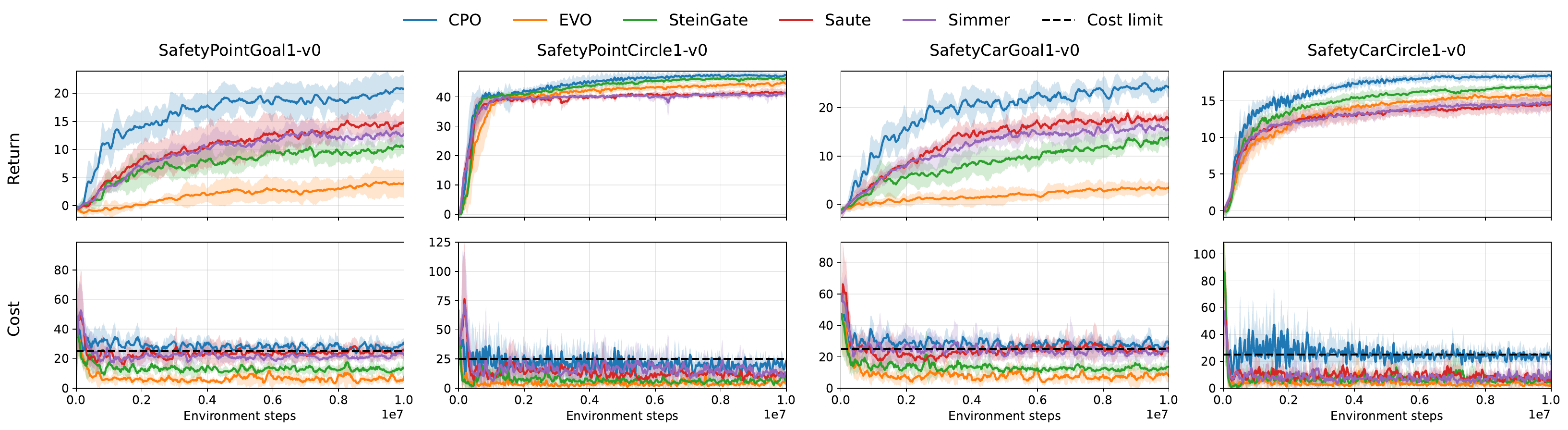}
    \caption{Results comparing SteinGate and baselines on Safety Gym tasks. The x-axis denotes the total training steps, and the y-axis reports either the average return or the average constraint value. Solid curves indicate the mean across runs, with shaded regions representing one standard deviation. The dashed horizontal line denotes the constraint threshold of 25.}
    \label{fig:main_results}
\end{figure*}

\noindent \textbf{SteinGate vs. Baselines.} Figure \ref{fig:main_results} illustrates the training dynamics across the Safety Gymnasium suite. SteinGate exhibits a robust ability to converge to the feasible set, driving episodic costs below the threshold while sustaining continuous reward improvement. In contrast, CPO prioritizes reward maximization at the expense of safety, resulting in frequent and sustained violations (particularly in {CarGoal1}). Conversely, EVO demonstrates strict safety adherence but suffers from significant conservatism, often plateauing at suboptimal return levels. This confirms that relying solely on tail modeling (EVO) or expectation constraints (CPO) leads to suboptimal trade-offs, whereas SteinGate effectively identifies policies that satisfy the constraint while achieving strong task performance. Table \ref{tab:tail_feas_feas_ret} quantifies this advantage over the final training phase. On the {PointGoal1} and {CarGoal1} tasks, SteinGate matches the near-perfect tail feasibility of {EVO} ($>0.98$) while delivering substantially higher feasible returns (e.g., nearly $3\times$ higher on {PointGoal1} and $4\times$ on {CarGoal1}). This indicates that SteinGate resolves the over-conservatism inherent in EVO's extreme-value approach. On {PointCircle1}, SteinGate outperforms CPO by achieving equivalent returns with strictly higher feasibility ($0.99$ vs. $0.77$), eliminating the safety violations common in primal-dual methods. Finally, compared to state-augmentation baselines ({Saute} and {Simmer}), SteinGate consistently achieves superior feasible returns, demonstrating that direct policy regulation is more efficient than curriculum-based or augmented-state approaches. We observe similar trends on the velocity benchmarks and against additional baselines; results are reported in Appendix~\ref{sec:additional_results}.

\begin{table}[t]
\centering
\small
\caption{Tail feasibility and feasible return across Safety Gymnasium tasks. Tail feasibility is computed over the final 20\% of training; feasible return is the episodic return conditioned on satisfying the cost limit.}
\label{tab:tail_feas_feas_ret}
\setlength{\tabcolsep}{2pt}
\begin{tabular}{llcc}
\toprule
Environment & Algorithm & Tail Feasibility $\uparrow$ & Feasible Return $\uparrow$ \\
\midrule

\multirow{5}{*}{PointGoal1}
& CPO           & 0.258 & 19.21 \\
& EVO           & 0.988 & 3.68  \\
& Saute     & 0.563 & 14.26 \\
& Simmer & 0.858 & 12.33 \\ 
& \textbf{SteinGate}     & \textbf{0.987} & \textbf{9.77}  \\
\midrule

\multirow{5}{*}{PointCircle1}
& CPO           & 0.775 & 47.16 \\
& EVO           & 0.990 & 44.27 \\
& Saute     & 0.992 & 41.25 \\
& Simmer & 0.915 & 40.80 \\
& \textbf{SteinGate}     & \textbf{0.992} & \textbf{46.15} \\
\midrule

\multirow{5}{*}{CarGoal1}
& CPO           & 0.257 & 22.86 \\
& EVO           & 0.997 & 3.23  \\
& Saute     & 0.575 & 17.54 \\
& Simmer & 0.698 & 15.62 \\
& \textbf{SteinGate}     & \textbf{0.998} & \textbf{12.65} \\
\midrule

\multirow{5}{*}{CarCircle1}
& CPO           & 0.578 & 18.24 \\
& EVO           & 0.995 & 15.64 \\
& Saute     & 0.998 & 14.31 \\
& Simmer & 0.977 & 14.51 \\
& \textbf{SteinGate}     & \textbf{0.997} & \textbf{16.77} \\
\bottomrule
\end{tabular}
\end{table}

\textbf{Sensitivity to Cost Thresholds.}
Figure~\ref{fig:clim_sensitivity} illustrates SteinGate’s behavior under varying cost limits $(\kappa \in \{10, 25, 40\})$. As the constraint is tightened, SteinGate adaptively regulates its policy to maintain episodic cost near the prescribed threshold while modulating the achievable return. For moderate thresholds ($\kappa$ = 25, 40), the learned policies consistently stabilize strictly below the cost limits. Under the highly restrictive setting ($\kappa$ = 10), the episodic cost fluctuates tightly around the threshold, indicating operation at the boundary of the feasible set with only marginal violations. Concurrently, episodic return increases monotonically as the cost limit is relaxed, demonstrating that SteinGate effectively capitalizes on the expanded safety budget to maximize task performance. Collectively, these results demonstrate that SteinGate scales smoothly with stringency of cost constraint, maintaining a stable balance between reward maximization and safety.

\begin{figure}[h]
    \centering
    \includegraphics[width=\linewidth]{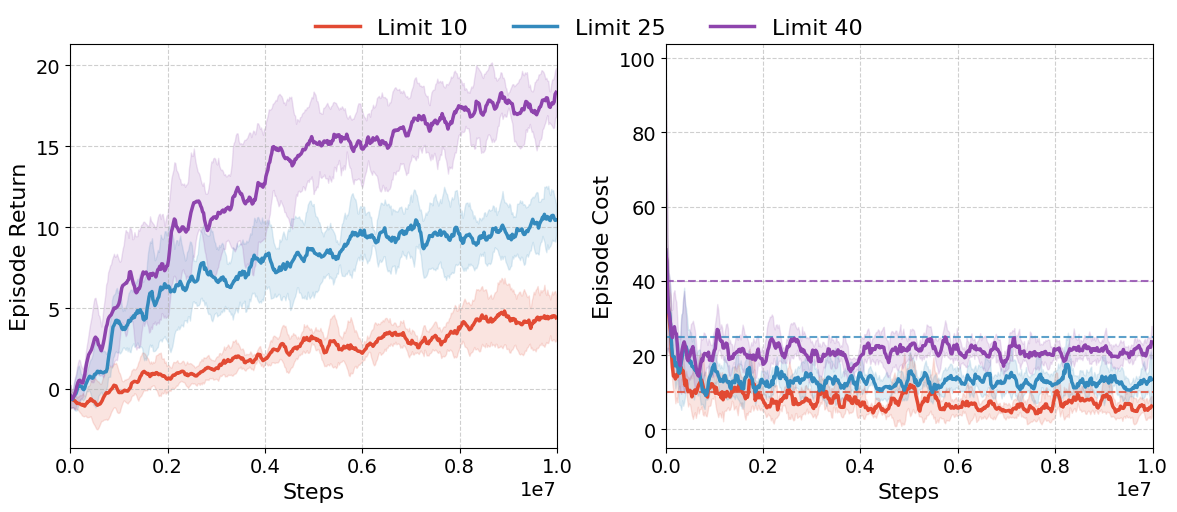}
    \caption{Cost-limit sensitivity of SteinGate on SafetyPointGoal1-v0. The plots show episodic return (left) and episodic cost (right) during training under different cost limits ($\kappa$ =10, 25, 40). Dashed lines indicate the corresponding cost thresholds.}
    \label{fig:clim_sensitivity}
\end{figure}

\textbf{Sensitivity to Reference Distribution.} 
Figure~\ref{fig:safe_ref} examines the impact of the interior reference distribution on \textsc{PointGoal1}, comparing the default $\mathrm{Beta}(2,5)$ against a conservative $\mathrm{Beta}(1,10)$ and a logit-normal alternative. Across all three specifications, the costs remain relatively consistent, suggesting that the mechanism for constraint enforcement is not highly sensitive to the specific choice of reference family on this task. The primary difference lies in the efficiency of reward acquisition: the conservative $\mathrm{Beta}(1,10)$ induces a more cautious utilization of the safety budget, resulting in lower task performance. Conversely, both the default $\mathrm{Beta}(2,5)$ and the logit-normal references converge to higher final returns. These results indicate that while the reference distribution modulates the degree of conservatism, SteinGate can maintain effective constraint satisfaction across different reasonable parametric choices. The observed robustness across different reference families suggests that SteinGate depends primarily on the specification of an admissible safety profile rather than on a particular parametric form. More conservative references induce correspondingly conservative policies, while comparable safety behavior across multiple reference families indicates that the framework is not tied to a specific distributional assumption.

\begin{figure}[h]
    \centering
    \includegraphics[width=\linewidth]{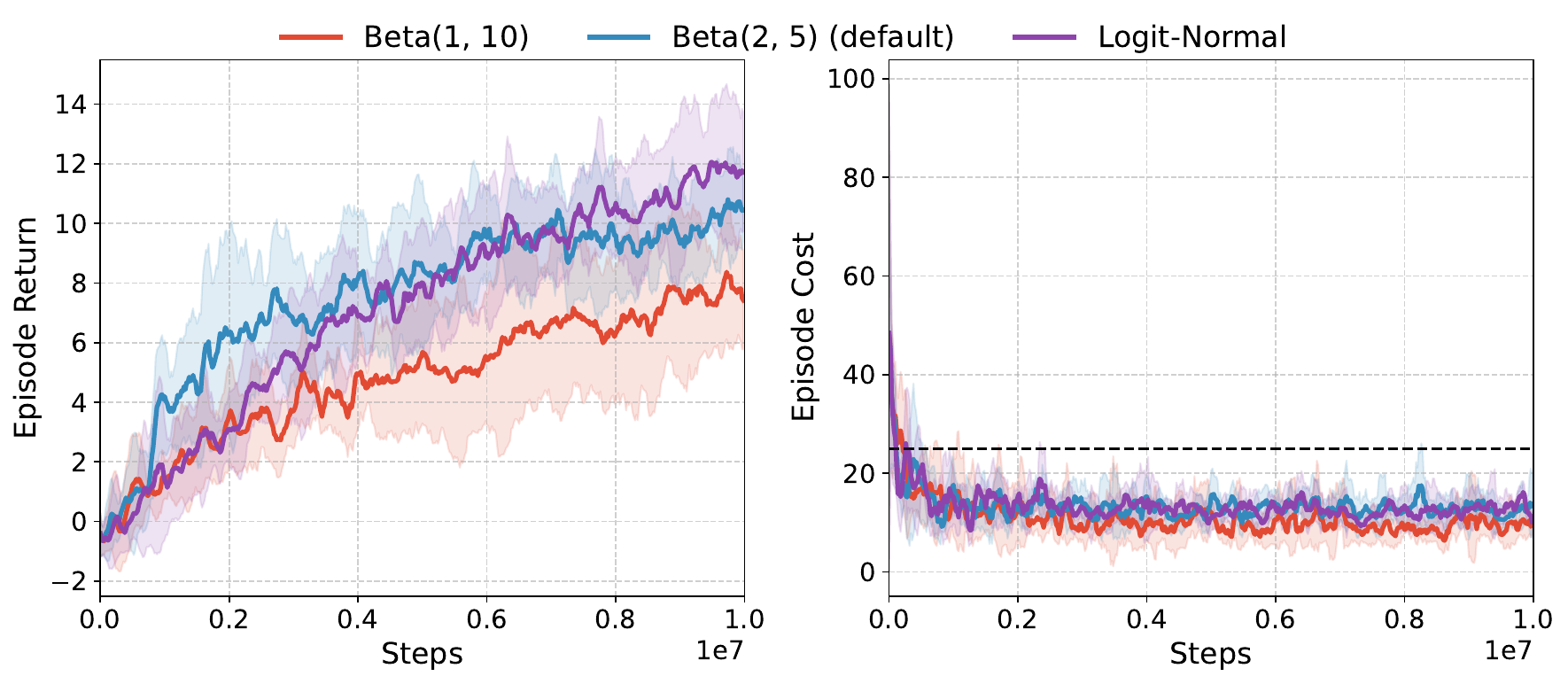}
    \caption{SteinGate safe reference ablation on SafetyPointGoal1-v0. Episodic return (left) and episodic cost (right) during training for three interior reference distributions: Beta(1,10), Beta(2,5), and logit-normal.}
    \label{fig:safe_ref}
\end{figure}

\textbf{Sensitivity to Sample Size.} \label{pg:sample_size_ablation}
Finally, we investigate the robustness of SteinGate to the number of cost samples available for estimating the certificate, in Figure~\ref{fig:safe_samples}. Unlike tail-modeling approaches (e.g., EVO) that often require off-policy data augmentation to accurately fit distributions, SteinGate operates strictly on-policy. We evaluated SteinGate on {CarGoal1} with effective batch sizes of $N \in \{8, 16, 32\}$ samples per epoch. Despite the scarcity of data in the low-sample regime, SteinGate consistently enforced safety across all settings. We observed a natural trade-off: larger batches ($N$=32) resulted in slightly more conservative policies, while smaller batches ($N$=8) exhibited higher variance but remained safe. These results demonstrate that the SteinGate mechanism is highly sample-efficient and effective even within the limited interaction budgets of standard on-policy training loops.

This behavior is consistent with the role of KSD as a sample-based distributional certificate rather than a parametric tail estimator: the method does not require fitting a rare-event model from a large buffer of exceedances, but instead monitors whether the observed rollout costs remain consistent with the prescribed reference.

\begin{figure}[h]
    \centering
    \includegraphics[width=\linewidth]{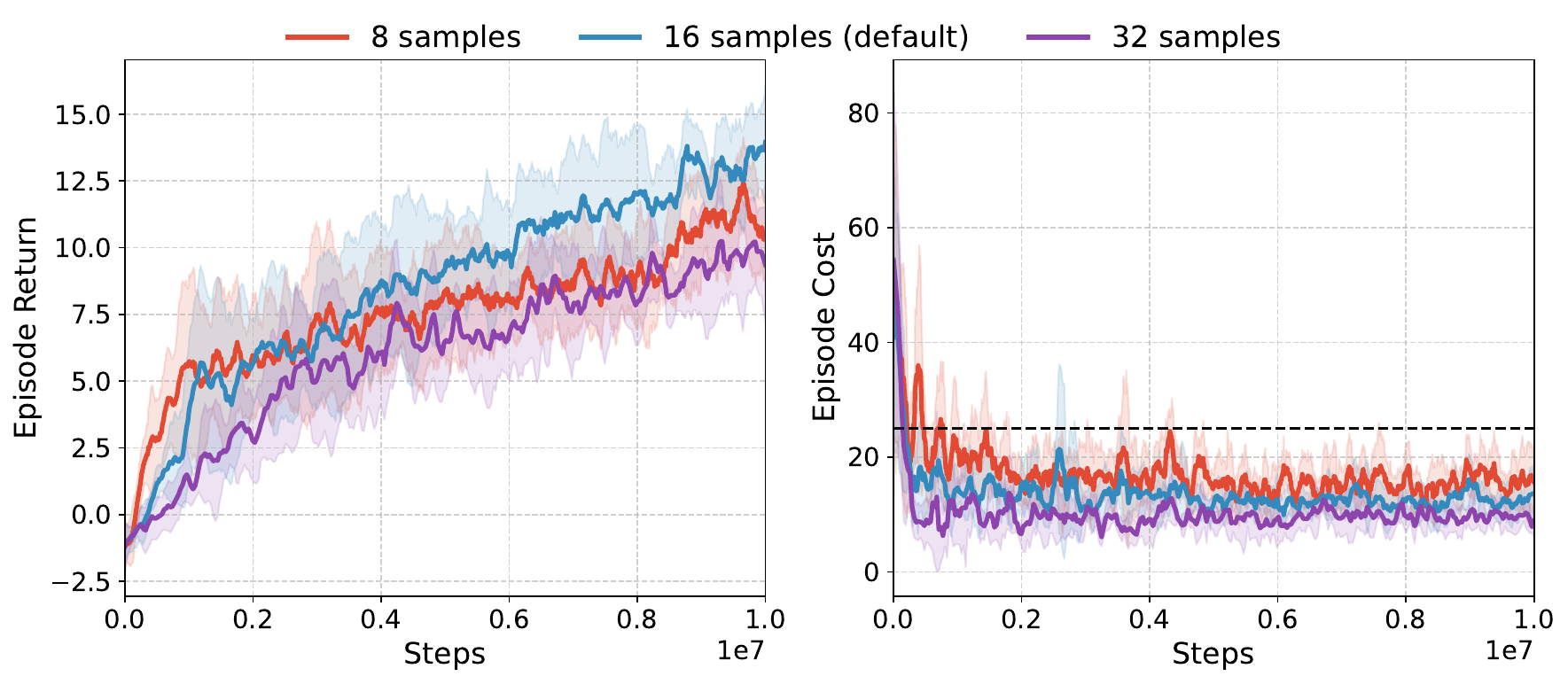}
    \caption{SteinGate sample size ablation on SafetyCarGoal1-v0. Episodic return (left) and episodic cost (right) during training for three sizes: 8, 16, and 32.}
    \label{fig:safe_samples}
\end{figure}

\section{Related Work}
\label{sec:related_work}
\noindent \textbf{Expectation-Based Safety.}
Safe RL is commonly formulated as a CMDP problem, where safety is enforced by constraining the \emph{expected} cumulative cost. 
Representative methods include trust-region methods like constrained policy optimization (CPO) \citep{achiam2017constrained}, FISOR \citep{ZheLiYu_24}, SEVPO \citep{ZhuYinWei_26}, projection-based approaches such as PCPO \citep{yang2020projection} and CUP \citep{yang2022constrained}, first-order constrained optimization in policy space (FOCOPS) \citep{zhang2020first}, multi-timescale Lagrangian methods such as RCPO \citep{tessler2018reward} and PID-Lagrangian \citep{stooke2020responsive}, and reductions to unconstrained RL\citep{O3SRL,CAPS}. While these approaches provide control over expected cost, they do not directly regulate the distribution of outcomes; policies with identical means may exhibit severe tail violations.

\noindent \textbf{Risk-Sensitive and Distributional Safety.}
To address this limitation, several works impose constraints on distributional functionals of cost. Common choices include Value-at-Risk (VaR) and Conditional Value-at-Risk (CVaR) \citep{rockafellar2002conditional}. In RL, CVaR constraints have been implemented via state augmentation and Lagrangian optimization \citep{chow2015risk, chow2018risk}, while actor--critic variants such as WCSAC \citep{yang2021wcsac} and QCPO \citep{jung2022quantile} estimate quantiles or parametric cost distributions. Extreme value theory has also been applied to model high-cost exceedances using generalized Pareto distributions, as in extreme value policy optimization (EVO) \citep{gao2026extremevaluepolicyoptimization} and related approaches \citep{ns2024extreme}. These methods explicitly target tail behavior but rely on estimating rare-event statistics.  
Alternatively, Sauté RL \citep{sootla2022saute} and Simmer \citep{sootla2022effects} enforce safety almost surely via budget state augmentation.

\noindent \textbf{Stein Discrepancies in RL.} Kernelized Stein discrepancy (KSD) \citep{liu2016kernelized} offers a nonparametric measure of distributional deviation, previously applied in RL for detecting transition shifts \citep{chakraborty2023posterior} or inducing pessimism in offline settings \citep{koppel2024information}. In contrast, we employ KSD to certify tail-sensitive safety in online RL via a boundary-aware formulation that monitors the cost distribution during optimization.

\section{Conclusion}
This paper introduced and studied SteinGate, a distributional safety framework for tail-sensitive constrained reinforcement learning. By leveraging kernelized Stein discrepancy, SteinGate certifies safety through deviation from a designated hybrid reference distribution, avoiding direct estimation of rare-event statistics. The boundary-aware formulation accounts for probability mass induced by cost clipping and episodic termination, enabling theoretically-sound safety monitoring under non-stationary policy updates. Empirical results across multiple benchmarks demonstrate that SteinGate reduces violation frequency and severity during training while maintaining strong task performance. These findings suggest that distributional certification provides a principled and practical approach to mitigating extreme violations in safe reinforcement learning.

\clearpage

\section{Acknowledgments}

The authors gratefully acknowledge the in part support by USDA-NIFA funded AgAID Institute
award 2021-67021-35344 and NSF ECCS-2534263 grant. The views expressed are those of the authors and do not reflect the official policy or position of the USDA-NIFA and the NSF.

\bibliography{uai2026-template,hhw}

\newpage

\onecolumn

\title{SteinGate: Tail-Sensitive Safe Reinforcement Learning via Stein Discrepancy\\(Supplementary Material)}
\maketitle

\appendix

\section{Background of Stein's Method and Kernelized Stein Discrepancy }
\label{ap:stein}

Stein's method yields identities and discrepancy measures between distributions. 
\begin{definition}
    Assume that $\mathcal{X}$ is a subset of $\mathbb{R}^d$ and $p(x)$ a smooth density whose support is $\mathcal{X}$. The (Stein) score function of $p$ is defined as
\begin{equation}
s_p \;=\; \nabla_x \log p(x) \;=\; \frac{\nabla_x p(x)}{p(x)}\,.
\end{equation}
\end{definition}
\begin{definition}
We say that a function $f:\mathcal{X}\to\mathbb{R}$ is in the \textbf{Stein class} of $p$, denoted as $f \in \mathcal S(p)$,
if $f$ is smooth and satisfies
\begin{equation}\label{eq:stein-class}
\int_{x\in\mathcal{X}} \nabla_x\!\big(f(x)\,p(x)\big)\,dx \;=\; 0.
\end{equation}
\end{definition}
\begin{definition}
The Stein operator of $p$ is a linear operator acting on the Stein class of $p$, defined as
\begin{equation}
\mathcal{A}_{p} f(x) \;=\; s_{p}(x)\, f(x)^{\top} \;+\; \nabla_{x} f(x).
\end{equation}
\end{definition}
\begin{lemma}[Stein's Identity]
Assume $p(x)$ is a smooth density supported on $\mathcal{X}$, then
\begin{equation}
\mathbb{E}_{p}\!\left[\mathcal{A}_{p} f(x)\right]
=
\mathbb{E}_{p}\!\left[s_{p}(x)\, f(x)^{\top} + \nabla f(x)\right]
= 0,
\end{equation}
for any $f$ that is in the Stein class of $p$.
\end{lemma}

Let $k(x,x')$ be a positive definite kernel. The spectral decomposition of $k(x,x')$,
as implied by Mercer's theorem, is defined as
\begin{equation}\label{eq:mercer}
k(x,x')=\sum_{j}\lambda_j e_j(x)e_j(x'),
\end{equation}
where $\{e_j\}$ and $\{\lambda_j\}$ are the orthonormal eigenfunctions and positive eigenvalues of
$k(x,x')$, respectively, satisfying
\[
\int e_i(x)e_j(x)\,dx=\mathbb{I}[i=j],\qquad \forall\, i,j.
\]

For a positive definite kernel $k(x,x')$, its related RKHS $\mathcal{H}$ comprises linear combinations
of its eigenfunctions, i.e.,
\[
f(x)=\sum_{j} f_j e_j(x)
\quad \text{with} \quad
\sum_{j}\frac{f_j^2}{\lambda_j}<\infty,
\]
endowed with an inner product
\[
\langle f,g\rangle_{\mathcal{H}}=\sum_{j}\frac{f_j g_j}{\lambda_j}
\quad \text{between} \quad
f(x)=\sum_{j} f_j e_j(x),\ \ g(x)=\sum_{j} g_j e_j(x).
\]
Thus this Hilbert space is equipped with a norm $\|f\|_{\mathcal{H}}$ where
\[
\|f\|_{\mathcal{H}}^2=\langle f,f\rangle_{\mathcal{H}}=\sum_{j}\frac{f_j^2}{\lambda_j}.
\]

One can verify that $k(x,\cdot)\in\mathcal{H}$ and satisfies the important \emph{reproducing} property,
\[
f(x)=\langle f,k(\cdot,x)\rangle_{\mathcal{H}},
\qquad
k(x,x')=\langle k(\cdot,x),k(\cdot,x')\rangle_{\mathcal{H}}.
\]
Every positive definite kernel $k$ defines a unique RKHS for which $k$ is a reproducing kernel.
\begin{definition}
A kernel $k(x,x')$ is said to be \emph{integrally strictly positive definite}, if for any function $g$
that satisfies $0<\|g\|_2^2<\infty$,
\begin{equation}\label{eq:ispd}
\int_{\mathcal{X}} g(x)\,k(x,x')\,g(x')\,dx\,dx' \;>\; 0.
\end{equation}
\end{definition}

\begin{definition}
The \textbf{kernelized Stein discrepancy} (KSD) $\mathbb{S}(p,q)$ between distribution $p$ and $q$ is defined as
\begin{equation}\label{eq:ksd}
\mathbb{S}(p,q)
=\mathbb{E}_{x,x'\sim p}\!\left[\delta_{q,p}(x)^{\top}\,k(x,x')\,\delta_{q,p}(x')\right],
\end{equation}
where $\delta_{q,p}(x)=s_q(x)-s_p(x)$ is the score difference between $p$ and $q$, and $x,x'$ are i.i.d.\ draws from $p(x)$.
\end{definition}

\begin{proposition}
Define $g_{p,q}(x)=p(x)\big(s_q(x)-s_p(x)\big)$.
Assume $k(x,x')$ is integrally strictly positive definite, and $p,q$ are continuous densities with
$\|g_{p,q}\|_2^2<\infty$. Then $\mathbb{S}(p,q)\ge 0$ and $\mathbb{S}(p,q)=0$ if and only if $p=q$.
\end{proposition}

\begin{definition}
A kernel $k(x,x')$ is said to be in the Stein class of $p$ if $k(x,x')$ has continuous second-order
partial derivatives, and both $k(x,\cdot)$ and $k(\cdot,x)$ are in the Stein class of $p$ for any fixed $x$.
\end{definition}

It is easy to check that the RBF kernel
\[
k(x,x')=\exp\!\left(-\frac{1}{2h^{2}}\|x-x'\|_2^{2}\right)
\]
is in the Stein class for smooth densities supported on $\mathcal{X}=\mathbb{R}^{d}$.

\begin{proposition}
If $k(x,x')$ is in the Stein class of $p$, so is any $f\in\mathcal{H}$.
\end{proposition}

\begin{theorem}
Assume $p$ and $q$ are smooth densities and $k(x,x')$ is in the Stein class of $p$.
Define
\begin{align}
u_q(x,x')
&= s_q(x)^{\top} k(x,x') s_q(x')
   + s_q(x)^{\top}\nabla_{x'} k(x,x')
   + \nabla_{x} k(x,x')^{\top} s_q(x')
   + \operatorname{trace}\!\big(\nabla_{x,x'} k(x,x')\big).
\end{align}
Then
\begin{equation}
\mathbb{S}(p,q)=\mathbb{E}_{x,x'\sim p}\!\left[u_q(x,x')\right].
\end{equation}
\end{theorem}

\section{Theoretical Proofs}

\subsection{Proof of Lemma \ref{lem:interior_stein_domination}} \label{ap:proof_of_stein_certi}

First, we show that $B(h,p)$ has an upper bound for specific test-function class $\mathcal H$.

\begin{lemma}\label{lem:A2-1d}
Let $p$ be a density distribution such that $p\in C^1((0,1))$ and $p(c)>0$ for all $c\in(0,1)$.
Let $F(c):=\int_0^c p(t)\,dt$ and $s_p(c):=\partial_c\log p(c)$.
Let $\mathcal A_p$ be the (one-dimensional) Langevin--Stein operator
\[
(\mathcal A_p f)(c):= f'(c)+ s_p(c) f(c).
\]
Let $\mathcal H$ be any test-function class satisfying $\mathcal H\subseteq\mathcal H_0$ where
\[
\mathcal H_0:=\{h:[0,1]\to[0,1], \text{nondecreasing} \}
\]
Assume the following two quantities are finite:
\begin{equation}\label{eq:KpMp}
\begin{aligned}
    &K_p:=\sup_{c\in(0,1)}\frac{\min\{F(c),\,1-F(c)\}}{p(c)}<\infty,
\\
&M_p:=\sup_{c\in(0,1)}|w(c)\,s_p(c)|<\infty.
\end{aligned}
\end{equation}
Denote 
\begin{equation}
    \mathcal P_{\mathrm{int}}(\mathcal H,\mathcal F)\coloneqq\{p\mid f_h\in\mathcal F\cap \mathcal S(p), \forall \ h\in\mathcal H\}
    \label{eq:def-Pint}
\end{equation}
Then $p\in\mathcal P_{\mathrm{int}}(\mathcal H,\mathcal F_{\mathrm{int}})$.
More precisely, for any $h\in\mathcal H$ and $g(c):=h(c)-\mathbb E_{C\sim p}[h(C)]$,
the Stein equation $\mathcal A_p f_h=g$ admits a solution $f_h\in\mathcal F_{\mathrm{int}}$ given by
\begin{equation}\label{eq:fh-explicit}
f_h(c)=\frac{1}{p(c)}\int_0^c g(t)p(t)\,dt
      =-\frac{1}{p(c)}\int_c^1 g(t)p(t)\,dt,
\end{equation}
and its Stein factor satisfies
\begin{equation}\label{eq:stein-factor-bound-1d}
\begin{aligned}
B_{\mathrm{int}}(h,p):=\|f_h\|_{\mathcal F_{\mathrm{int}}}
&\le \Big(K_p+\tfrac14+M_pK_p\Big)\,\|g\|_\infty\\
&\le 2\Big(K_p+\tfrac14+M_pK_p\Big).
\end{aligned}
\end{equation}
\end{lemma}

\begin{proof}
Fix any $h\in\mathcal H$ and define $g(c)=h(c)-\mathbb E_p[h(C)]$.
Since $\|h\|_\infty\le 1$, we have $\|g\|_\infty\le 2$.

We solve the ODE $f'(c)+s_p(c)f(c)=g(c)$.
Multiplying by $p(c)$ and using $p'(c)=s_p(c)p(c)$ yields
\[
(p(c)f(c))' = g(c)p(c).
\]
Integrating from $0$ to $c$ and choosing the integration constant to be $0$ gives
\[
p(c)f_h(c)=\int_0^c g(t)p(t)\,dt,
\qquad
f_h(c)=\frac{1}{p(c)}\int_0^c g(t)p(t)\,dt.
\]
Since $\int_0^1 g(t)p(t)\,dt=0$ by centering, we also have the equivalent form
\[
f_h(c)=-\frac{1}{p(c)}\int_c^1 g(t)p(t)\,dt,
\]
which proves \eqref{eq:fh-explicit}.
We have
\[
p(c)f_h(c)=\int_0^c g(t)p(t)\,dt\to 0\quad \text{as }c\downarrow 0,
\]
and
\[
p(c)f_h(c)=-\int_c^1 g(t)p(t)\,dt\to 0\quad \text{as }c\uparrow 1.
\]
Moreover, $\mathbb E_p[|f_h'|+|s_pf_h|]<\infty$ follows from below.
Hence $f_h\in\mathcal S(p)$.

We first bound $\|f_h\|_\infty$ using both representations and $|g|\le \|g\|_\infty$:
\[
|f_h(c)|
\le \frac{\|g\|_\infty}{p(c)}
\min\Big\{\int_0^c p(t)\,dt,\ \int_c^1 p(t)\,dt\Big\}
=\|g\|_\infty\frac{\min\{F(c),1-F(c)\}}{p(c)}.
\]
Thus $\|f_h\|_\infty\le K_p\|g\|_\infty$.

Next, from the Stein equation, $f_h'(c)=g(c)-s_p(c)f_h(c)$, so
\[
|w(c)f_h'(c)|
\le |w(c)g(c)| + |w(c)s_p(c)|\,|f_h(c)|.
\]
Since $\sup_{c\in(0,1)}w(c)=1/4$, we have $\|wg\|_\infty\le \frac14\|g\|_\infty$.
Also $\|ws_p\|_\infty\le M_p$ and $\|f_h\|_\infty\le K_p\|g\|_\infty$.
Therefore,
\[
\|w f_h'\|_\infty
\le \tfrac14\|g\|_\infty + M_pK_p\|g\|_\infty.
\]
Combining the two bounds yields \eqref{eq:stein-factor-bound-1d}, and in particular
$\|f_h\|_{\mathcal F_{\mathrm{int}}}<\infty$, so $f_h\in\mathcal F_{\mathrm{int}}$.

For every $h\in\mathcal H$ there exists $f_h\in\mathcal F_{\mathrm{int}}$ such that
$\mathcal A_p f_h = h-\mathbb E_p[h]$ and $\|f_h\|_{\mathcal F_{\mathrm{int}}}<\infty$.
Also, $f_h\in\mathcal S(p)$, hence $\mathcal F_{\mathrm{int}}\subseteq\mathcal S(p)$ holds for this choice of solutions.
This matches Definition~\eqref{eq:def-Pint} and completes the proof.
\end{proof}

Then we start the proof of \cref{lem:interior_stein_domination}.

\begin{lemma}
    Let $p$ be a target distribution and $q_\pi$ be a proposal distribution on $\mathcal{C}$. Let $h: \mathcal{C} \to \mathbb{R}$ be a test function of interest. Suppose there exists a solution $f_h$ to the Stein Equation $\mathcal{A}_p f_h(C) = h(C) - \mathbb{E}_{C \sim p}[h(C)]$ such that $f_h \in \mathcal{F}$ and its norm is bounded by a constant $B(h,p)$, i.e., $\|f_h\|_{\mathcal{F}} \le B(h,p)$. Then, the expectation of $h$ under $q_\pi$ is bounded by:
    \begin{align}
        \mathbb{E}_{C \sim q_\pi}[h(C)] \;\le\; \mathbb{E}_{C \sim p}[h(C)] + B(h,p) \cdot SD(q_\pi\|p),
    \end{align}
    where $SD(q_\pi\|p):=\sup_{f_h\in\mathcal{F},\Vert f_h\Vert_\mathcal{F}\leq 1}
\left\vert\mathbb{E}_{C\sim q_\pi}\big[(\mathcal{A}_p f_h)(C)\big]\right\vert.$ 
\end{lemma}
\begin{proof}
   By the definition of the solution $f_h$ to the Stein equation, for any $C \in \mathcal{C}$, we have:$$h(C) - \mathbb{E}_{C \sim p}[h(C)] = \mathcal{A}_p f_h(C)$$
    
   Taking the expectation of both sides with respect to the distribution $q_\pi$:$$\mathbb{E}_{C \sim q_\pi}\left[ h(C) - \mathbb{E}_{C \sim p}[h(C)] \right] = \mathbb{E}_{C \sim q_\pi}[\mathcal{A}_p f_h(C)]$$Since $\mathbb{E}_{C \sim p}[h(C)]$ is a constant with respect to $q_\pi$, we can separate the terms on the left side:$$\mathbb{E}_{C \sim q_\pi}[h(C)] - \mathbb{E}_{C \sim p}[h(C)] = \mathbb{E}_{C \sim q_\pi}[\mathcal{A}_p f_h(C)]$$
   
   We seek to bound the term $\mathbb{E}_{C \sim q_\pi}[\mathcal{A}_p f_h(C)]$. Recall the definition of the Stein Discrepancy: $$SD(q_\pi\|p) \;\triangleq\; \sup_{f \in \mathcal{F}, \|f\|_{\mathcal{F}} \le 1} \mathbb{E}_{C \sim q_\pi}[\mathcal{A}_p f(C)]$$
   This implies that for any function $g \in \mathcal{F}$ with $\|g\|_{\mathcal{F}} \le 1$, the expectation $\mathbb{E}_{C \sim q_\pi}[\mathcal{A}_p g(C)]$ is at most $SD(q_\pi\|p)$. The specific solution $f_h$ has norm $\|f_h\|_{\mathcal{F}} \le B(h,p)$. We can construct a normalized function $g$ as:$$g(C) = \frac{f_h(C)}{B(h,p)}$$ Then the norm of this new function can be written as: $$\|g\|_{\mathcal{F}} = \left\| \frac{f_h}{B(h,p)} \right\|_{\mathcal{F}} = \frac{1}{B(h,p)} \|f_h\|_{\mathcal{F}} \le \frac{B(h,p)}{B(h,p)} = 1$$Since $\|g\|_{\mathcal{F}} \le 1$, $g$ is a valid candidate for the supremum in the definition of $SD(q_\pi\|p)$. Therefore:$$\mathbb{E}_{C \sim q_\pi}[\mathcal{A}_p g(C)] \le SD(q_\pi\|p)$$
   
   Then substituting $g(C)$ back in terms of $f_h(C)$ and using the linearity of the Stein operator $\mathcal{A}_p,$ and taking the expectation implies $\mathbb{E}$:$$\mathbb{E}_{C \sim q_\pi}\left[\mathcal{A}_p \left( \frac{f_h(C)}{B(h,p)} \right)\right] \le SD(q_\pi\|p)$$$$\frac{1}{B(h,p)} \mathbb{E}_{C \sim q_\pi}[\mathcal{A}_p f_h(C)] \le SD(q_\pi\|p)$$Multiplying both sides by the positive constant $B(h,p)$:$$\mathbb{E}_{C \sim q_\pi}[\mathcal{A}_p f_h(C)] \le B(h,p) \cdot SD(q_\pi\|p)$$
   
   Finally, substitute this inequality back we can easily obtain $$\mathbb{E}_{C \sim q_\pi}[h(C)] - \mathbb{E}_{C \sim p}[h(C)] \le B(h,p) \cdot SD(q_\pi\|p)$$Rearranging terms yields the final certificate:$$\mathbb{E}_{C \sim q_\pi}[h(C)] \le \mathbb{E}_{C \sim p}[h(C)] + B(h,p) \cdot SD(q_\pi\|p)$$
\end{proof}

\subsection{Stein Inequality}

\begin{proposition}[Hybrid Stein certificate with boundary atoms]
\label{prop:hybrid_stein_two_atoms}
Suppose $h\in\mathcal H_0$. Let $q$ denote $q_\pi$ as the rollout distribution under policy $\pi$.  
Assume Lemma \ref{lem:interior_stein_domination} holds on $(0,1)$ for the pair $(h,p_{\mathrm{int}})$.
If $\lambda\ge\max\{1,B_{\mathrm{int}}\}$,
then
\begin{equation}
\label{eq:hybrid_two_atoms}
\mathbb{E}_{C\sim q}[h(C)]
\le
U.
\end{equation}

\end{proposition}

\begin{proof}
By the atom and interior decomposition,
\[
\mathbb{E}_{C\sim q}[h(C)]
=
a_0(q)h_0 + a_1(q)h_1 + w(q)\ \mathbb{E}_{C\sim q_{\mathrm{int}}}[h(C)],
\]
and
\[
\mathbb{E}_{C\sim p_{\mathrm{ref}}}[h(C)]
=
a_0^\star h_0 + a_1^\star h_1 + w^\star\ \mathbb{E}_{C\sim p_{\mathrm{int}}}[h(C)].
\]

Applying Lemma \ref{lem:interior_stein_domination} to $q_{\mathrm{int}}$ yields
\begin{equation*}
\begin{aligned}
   \mathbb{E}_{C\sim q}[h(C)]
&\le
a_0(q)h_0 + a_1(q)h_1
+
w(q)\ \mathbb{E}_{C\sim p_{\mathrm{int}}}[h(C)]
+ w(q)\ 
B_{\mathrm{int}}
\cdot
\mathrm{SD}_{\mathrm{int}}\\
&=\mathbb E_{C\sim p_{\mathrm{ref}}}[h(C)]
+ h_0(a_0(q)-a_0^\star) + h_1(a_1(q)-a_1^\star) \\
 & \qquad{ \;-\; \mathbb E_{C\sim p_{\mathrm{int}}}[h(C)]\,\big((a_0(q)-a_0^\star)+(a_1(q)-a_1^\star)\big)}
+ w(q)\,B_{\mathrm{int}}\,\mathrm{SD}_{\mathrm{int}}\\
&=\mathbb{E}_{C\sim p_{\mathrm{ref}}}[h(C)]
+
(h_0-{\mathbb E_{C\sim p_{\mathrm{int}}}[h(C)]})(a_0(q)-a_0^\star)
+
(h_1-{\mathbb E_{C\sim p_{\mathrm{int}}}[h(C)]})(a_1(q)-a_1^\star)\\
&+
w(q)\cdot B_{\mathrm{int}}
\cdot
\mathrm{SD}_{\mathrm{int}}\\
&\le\mathbb{E}_{C\sim p_{\mathrm{ref}}}[h(C)]
+
{(a_0^\star-a_0(q))_+}
+
(a_1(q)-a_1^\star)_+
+
w(q)\cdot B_{\mathrm{int}}
\cdot
\mathrm{SD}_{\mathrm{int}}.
\end{aligned}
\end{equation*}

The last inequality uses monotonicity of $h$, i.e.\ $h_0\le \mathbb{E}_{C\sim p_{\mathrm{int}}}[h(C)]\le h_1$, which makes the term multiplying $(a_0(q)-a_0^\star)$ nonpositive (yielding the one-sided $(a_0^\star-a_0(q))_+$) and bounds the remaining centering factors by $1$ since $h\in[0,1]$. Finally, using the fact that
$\mathrm{SD}_{\mathrm{tot}}= (a_0(q)-a_0^\star)_+
+
(a_1(q)-a_1^\star)_+
+w(q)\cdot \mathrm{SD}_{\mathrm{int}}$
and the condition for $\lambda$ gives
\[
(a_0(q)-a_0^\star)_+
+
(a_1(q)-a_1^\star)_+
+
w(q)\cdot
B_{\mathrm{int}}
\cdot
\mathrm{SD}_{\mathrm{int}}
\le\lambda\cdot \mathrm{SD}_{\mathrm{tot}},
\]
which yields \eqref{eq:hybrid_two_atoms}.
\end{proof}

\subsection{Bound for Empirical Error of Stein Certificate}

Given $n$ i.i.d.\ rollouts (or approximately independent mini-batch samples) producing $\{c_i\}_{i=1}^n\sim q$,
define empirical atom masses
\begin{equation}
\hat a_0 := \frac{1}{n}\sum_{i=1}^n \mathbf{1}\{c_i\le\tau_0\},\qquad
\hat a_1 := \frac{1}{n}\sum_{i=1}^n \mathbf{1}\{c_i\ge 1-\tau_1\},\qquad
\hat w:=1-\hat a_0-\hat a_1.
\label{eq:pi-hat}
\end{equation}
Let $\{c^{\mathrm{int}}_i\}_{i=1}^m\subset\{c^{}_i\}_{i=1}^n$ be the interior samples with $m=n_{\mathrm{int}}$.
Let $\widehat{SD}_{\mathrm{int}}$ be the KSD estimator on interior samples
with V-statistic of the Stein kernel after clipping away from the boundary.
Define the empirical discrepancy
\begin{equation}
\widehat{SD}_{\textrm{disc}}
:=
(a_0^\star-\hat a_0)_+ + (\hat a_1-a_1^\star)_+,
\qquad
\widehat{SD}_{\textrm{tot}}
:=
\widehat{SD}_{\textrm{disc}} + \hat w\cdot \widehat{SD}_{\mathrm{int}},
\label{eq:sd-hat}
\end{equation}
where
\[
\widehat{SD}_{\mathrm{int}}
:= \widehat{\mathrm{KSD}}_V^2
:= \frac{1}{m^2}\sum_{i=1}^m\sum_{j=1}^m u_{p_{\mathrm{int}}}(C_i^{\mathrm{int}},C_j^{\mathrm{int}}).
\]
and the empirical certificate
\begin{equation}
U
:=
\E_{p_{\safe}}[h(C)] + \lambda\,\widehat{SD}_{\textrm{tot}}.
\label{eq:U-hat}
\end{equation}
Here we precompute $\E_{p_{\safe}}[h(C)]$ by Monte-Carlo Sampling
is a fixed constant independent of $q$.

\begin{assumption}[Bounded Stein kernel on the clipped interior]
\label{ass:bounded-kernel}
After interior clipping, the induced Stein kernel $u_{p_{\mathrm{int}}}(x,y)$
used by the KSD estimator satisfies $|u_{p_{\mathrm{int}}}(x,y)|\le M_u$ for all clipped $x,y$.
\end{assumption}

\begin{proposition}[Concentration of the empirical certificate]
\label{prop:certificate-concentration}
Assume Assumption~\ref{ass:bounded-kernel} holds.
Then for any $\delta\in(0,1)$, with probability at least $1-\delta$,
\begin{equation}
\big|{U}-U_{\text{stein}}\big|
\;\le\;
\lambda\Big(
\epsilon_{\mathrm{int}}(m,\delta)
+\epsilon_w(n,\delta)\cdot (1+SD_{\mathrm{int}})\Big)
,
\label{eq:U-conc}
\end{equation}
where
\begin{equation}
\epsilon_w(n,\delta):=2\sqrt{\frac{\log(6/\delta)}{2n}},
\qquad
\epsilon_{\mathrm{int}}(m,\delta):=
2\sqrt{2}\,M_u\sqrt{\frac{\log(6/\delta)}{m}}.
\label{eq:ksd2-conc}
\end{equation}
\end{proposition}

\begin{proof}

Each indicator in~\eqref{eq:pi-hat} is Bernoulli, hence by Hoeffding's inequality,
with probability at least $1-\delta/3$,
\[
|\hat a_0-a_0(q)|\le \sqrt{\frac{\log(6/\delta)}{2n}},
\qquad
|\hat a_1-a_1(q)|\le \sqrt{\frac{\log(6/\delta)}{2n}}.
\]
By the $1$-Lipschitz property of $(\cdot)_+$,
\[
\big|(a_0^\star-\hat a_0)_+-(a_0^\star-a_0(q))_+\big|\le |\hat a_0-a_0(q)|,
\quad
\big|(\hat a_1-a_1^\star)_+-(a_1(q)-a_1^\star)_+\big|\le |\hat a_1-a_1(q)|.
\]
Thus,
\[
|\widehat{SD}_{\textrm{disc}}-SD_{\textrm{disc}}(q\|p_{\safe})|
\le
|\hat a_0-a_0(q)|+|\hat a_1-a_1(q)|
\le
2\sqrt{\frac{\log(6/\delta)}{2n}}.
\]
Moreover, $\hat w-w(q)=-(\hat a_0-a_0(q))-(\hat a_1-a_1(q))$ implies
$|\hat w-w(q)|\le 2\sqrt{\frac{\log(6/\delta)}{2n}}$, yielding the first relation in~\eqref{eq:ksd2-conc}.

Consider the V-statistic estimator
$\widehat{SD}_{\mathrm{int}}=\frac{1}{m^2}\sum_{i,j=1}^m u_{p_{\mathrm{int}}}(c_i^{\mathrm{int}},c_j^{\mathrm{int}})$.
Changing one sample $c_i^{\mathrm{int}}$ affects at most $(2m-1)$ summands.
Since $|u_{p_{\mathrm{int}}}|\le M_u$, each affected summand changes by at most $2M_u$,
hence the bounded-differences constant satisfies
\[
c_i \le \frac{(2m-1)\cdot 2M_u}{m^2} \le \frac{4M_u}{m}.
\]
By McDiarmid's inequality, with probability at least $1-\delta/3$,
\[
\big|\widehat{SD}_{\mathrm{int}}-\mathbb{E}[\widehat{SD}_{\mathrm{int}}]\big|
\le \sqrt{\frac{\sum_{i=1}^m c_i^2}{2}\log\frac{6}{\delta}}
\le \sqrt{\frac{m\cdot(16M_u^2/m^2)}{2}\log\frac{6}{\delta}}
=
2\sqrt{2}\,M_u\sqrt{\frac{\log(6/\delta)}{m}},
\]
which yields the second relation in~\eqref{eq:ksd2-conc}.

Using~\eqref{eq:sd-hat}, triangle inequality and the fact that $\hat w\le 1$,
\[
|\widehat{SD}_{\textrm{tot}}-SD_{\textrm{tot}}|
\le
|\widehat{SD}_{\textrm{disc}}-SD_{\textrm{disc}}|
+
|\hat w-w|\cdot SD_{\mathrm{int}}
+
|\widehat{SD}_{\mathrm{int}}-SD_{\mathrm{int}}|.
\]
Multiplying by $\lambda$ gives~\eqref{eq:U-conc}.
A union bound over the three events completes the proof.
\end{proof}

\begin{lemma}
\label{lem:sd_lcb}
Assume the concentration results in Proposition \ref{prop:certificate-concentration} hold.
Define the computable lower confidence bound 
\begin{equation}
\label{eq:sd_lcb_def}
\underline{SD}_{\rm tot}
~:=~
\widehat{SD}_{\rm tot}
-\epsilon_{\rm int}(m,\delta)
-\epsilon_w(n,\delta)\Big(1+\widehat{SD}_{\rm int}+\epsilon_{\rm int}(m,\delta)\Big),
\qquad
\underline{SD}_{\rm tot}^+ := \max\{\underline{SD}_{\rm tot},0\}.
\end{equation}
Then with probability at least $1-\delta$,
\begin{equation}
\label{eq:sd_lcb_claim}
SD_{\rm tot} \;\ge\; \underline{SD}_{\rm tot}^+.
\end{equation}
\end{lemma}

\begin{proof}
By Proposition \ref{prop:certificate-concentration}, with probability at least $1-\delta$,
\begin{align}
\label{eq:conc_tot}
\big|\widehat{SD}_{\rm tot}-SD_{\rm tot}\big|
&\le
\epsilon_{\rm int}(m,\delta)
+\epsilon_w(n,\delta)\big(1+SD_{\rm int}\big),
\\
\label{eq:conc_int}
\big|\widehat{SD}_{\rm int}-SD_{\rm int}\big|
&\le
\epsilon_{\rm int}(m,\delta).
\end{align}
Let $\mathcal{E}$ be the event on which both~\eqref{eq:conc_tot} and~\eqref{eq:conc_int} hold.
Then $\Pr(\mathcal{E})\ge 1-2\delta$.

On $\mathcal{E}$, inequality~\eqref{eq:conc_tot} implies
\begin{equation}
\label{eq:sd_tot_lb_step1}
SD_{\rm tot}
~\ge~
\widehat{SD}_{\rm tot}
-\epsilon_{\rm int}(m,\delta)
-\epsilon_w(n,\delta)\big(1+SD_{\rm int}\big).
\end{equation}
Moreover, from~\eqref{eq:conc_int} we have on $\mathcal{E}$ that
\begin{equation}
\label{eq:sd_int_ub}
SD_{\rm int}
~\le~
\widehat{SD}_{\rm int}+\epsilon_{\rm int}(m,\delta).
\end{equation}
Substituting~\eqref{eq:sd_int_ub} into~\eqref{eq:sd_tot_lb_step1} yields
\begin{align}
SD_{\rm tot}
&\ge
\widehat{SD}_{\rm tot}
-\epsilon_{\rm int}(m,\delta)
-\epsilon_w(n,\delta)\Big(1+\widehat{SD}_{\rm int}+\epsilon_{\rm int}(m,\delta)\Big)
~=~
\underline{SD}_{\rm tot}.
\end{align}
Since $SD_{\rm tot}\ge 0$ by definition, we further obtain
$SD_{\rm tot}\ge \max\{\underline{SD}_{\rm tot},0\}=\underline{SD}_{\rm tot}^+$.
This proves~\eqref{eq:sd_lcb_claim}.

\end{proof}

\subsection{Proof of Theorem \ref{thm:stein-cpo}}
\label{ap:stein-cpo}

Consider an episodic CMDP $\mathcal{M}$ with horizon $H$, state space $\mathcal{S}$, action space $\mathcal{A}$,
transition kernel $P(\cdot\mid s,a)$, and per-step cost $c:\mathcal{S}\times\mathcal{A}\to[0,1]$.
For a trajectory $\tau=(s_0,a_0,\dots,s_{H-1},a_{H-1})$ generated by policy $\pi$,
define the (unnormalized) cumulative cost
\[
C(\tau)\;:=\;\sum_{t=0}^{H-1} c(s_t,a_t).
\]
Let $h:\mathbb{R}\to\mathbb{R}$ be a measurable test function (e.g., sigmoid).
Our target functional is
\[
C_h^{\pi}(\mu_0)\;:=\;\E_{\tau\sim\pi}\big[h(C(\tau))\big].
\]
This is not additive in $(s_t,a_t)$ in general, hence performance difference lemma does not apply directly.

\paragraph{Augmented CMDP.}
We define an augmented CMDP $\bar{\mathcal{M}}$ whose state includes the running cumulative cost.
Let the augmented state space be
\[
\bar{\mathcal{S}} := \mathcal{S}\times \mathbb{R},
\qquad
\bar{s}_t := (s_t, z_t),
\qquad
\bar{\mu}_0:=({\mu}_0, z_0)
\]
where the auxiliary coordinate $z_t$ evolves deterministically as
\[
z_0:=0,\qquad z_{t+1} := z_t + c(s_t,a_t),\quad t=0,\dots,H-1.
\]
The augmented transition kernel is therefore
\[
\bar{P}\big((s',z')\mid (s,z),a\big)
=
P(s'\mid s,a)\cdot \mathbf{1}\{z'=z+c(s,a)\}.
\]

Define a terminal cost on the augmented state at time $H$:
\[
\bar{c}_H(\bar{s}_H) := h(z_H),
\]
and define per-step costs for $t=0,\dots,H-1$ as identically zero:
\[
\bar{c}_t(\bar{s}_t,a_t):=0,\qquad t=0,\dots,H-1.
\]
Let $\bar{V}_c^\pi(\bar\mu_0)$ denote the expected total cost under current policy $\pi$ in $\bar{\mathcal{M}}$:
\[
\bar{V}_c^\pi(\bar\mu_0) := \E_{\tau\sim \pi}\Big[\sum_{t=0}^{H-1}\bar{c}_t(\bar{s}_t,a_t) + \bar{c}_H(\bar{s}_H)\Big].
\]
By construction,
\begin{equation}
\label{eq:Jh_equals_augJ}
\bar{V}_c^\pi(\bar\mu_0)=\E_{\tau\sim\pi}[h(z_H)]=\E_{\tau\sim\pi}[h(C(\tau))]=C_h^\pi(\mu_0).
\end{equation}

\paragraph{Applying Performance Difference Lemma on the augmented CMDP.}
Now we can apply standard trust-region performance-difference bounds on $\bar{\mathcal{M}}$,
because $\bar{J}(\pi)$ is an expected return in the augmented CMDP.

Let $\pi,\pi'$ be two policies on $\mathcal{S}$ (equivalently on $\bar{\mathcal{S}}$ by ignoring $z$ in the action selection).
Define the value and $Q$ functions for the augmented MDP with respect to policy $\pi$:
\[
C_h^\pi(\bar{s}_t) := \E_{\pi}\Big[\bar{c}_H(\bar{s}_H)\,\big|\,\bar{s}_t\Big],\qquad
Q_h^\pi(\bar{s}_t,a) := \E_\pi\Big[\bar{c}_H(\bar{s}_H)\,\big|\,\bar{s}_t,a_t=a\Big],
\]
and the corresponding advantage
\[
A_h^\pi(\bar{s}_t,a):=Q_h^\pi(\bar{s}_t,a)-C_h^\pi(\bar{s}_t).
\]
Let $\bar{d}^\pi_t$ denote the state distribution of $\bar{s}_t$ under $\pi$ in the augmented CMDP
and $\bar{d}^\pi:=\frac1H\sum_{t=0}^{H-1}\bar{d}^\pi_t$.

By augmentation, we can apply trust-region performance-difference bounds on 
$\bar {\mathcal M}$
. We therefore consider the conservative surrogate program whose constraint can be rewritten as:
\begin{equation}
\begin{aligned}
    \label{eq:stein-tr}
&\underbrace{
C_h^{\pi_k}(\bar \mu_0)
+
\sum_{t=0}^{H-1}\mathbb{E}_{\bar s\sim \bar d^{\pi_k},a\sim\pi}\big[A^{\pi_k}_h(\bar s,a)\big]
}_{\text{CPO}}
+
\underbrace{m_k}_{\text{Stein margin}}
\le \varepsilon,
\;\;
\bar D_{KL}(\pi\|\pi_k)\le \delta,
\end{aligned}
\end{equation}
where the Stein margin is
\[
m_k\triangleq \lambda_k\cdot \underline{SD}_{\mathrm{tot}}^+(q_{\pi_k}\|p_{\mathrm{ref}})\;\;\ge 0.
\]
\medskip
\noindent
\begin{lemma}\citep{achiam2017constrained}
    For any $\pi,\pi'$,
\begin{equation}
\label{eq:achiam_aug}
C_h^{\pi'}(\bar \mu_0)-C_h^{\pi}(\bar \mu_0)
=
\bar V_c^{\pi'}(\bar \mu_0)-\bar V_c^{\pi}(\bar \mu_0)
\;\le\;
\sum_{t=0}^{H-1}\E_{\bar{s}\sim \bar{d}^\pi_t,\;a\sim\pi'(\cdot\mid s)}
\big[A_h^\pi(\bar{s},a)\big]
+
2H\,\epsilon_h^{\pi'}\cdot
\sum_{t=0}^{H-1}\E_{\bar{s}\sim \bar{d}^\pi_t}\Big[D_{\mathrm{TV}}\big(\pi'(\cdot\mid s),\pi(\cdot\mid s)\big)\Big]
\end{equation}
where the trust-region error term can be upper bounded by a total-variation penalty, e.g.,
\begin{equation}
\label{eq:tv_err}
D_{\mathrm{TV}}\big(\pi'(\cdot\mid s),\pi(\cdot\mid s)\big)=\frac12\sum_a\Big|\pi'(a\mid s)-\pi(a\mid s)\Big|,
\qquad
\epsilon_h^{\pi'}:=\max_{\bar{s}}\Big|\E_{a\sim\pi'(\cdot\mid s)}[A_h^\pi(\bar{s},a)]\Big|.
\end{equation}
\end{lemma}

Fix the current policy $\pi_k$ and let $\pi_{k+1}$ satisfy a trust-region constraint.
Applying \eqref{eq:achiam_aug}--\eqref{eq:tv_err} with $(\pi,\pi')=(\pi_k,\pi_{k+1})$ yields
\begin{equation}
\label{eq:one_step_aug_bound}
C_h^{\pi_{k+1}}(\bar \mu_0)
\;\le\;
C_h^{\pi_{k}}(\bar \mu_0)
+
\sum_{t=0}^{H-1}\E_{\bar{s}\sim \bar{d}^{\pi_k}_t,\;a\sim\pi_{k+1}(\cdot\mid s)}
\big[A_h^{\pi_k}(\bar{s},a)\big]
+
2H\,\epsilon_h^{\pi_{k+1}}\cdot
\sum_{t=0}^{H-1}\E_{\bar{s}\sim \bar{d}^\pi_k}\Big[D_{\mathrm{TV}}\big(\pi_{k+1}(\cdot\mid s),\pi_k(\cdot\mid s)\big)\Big].
\end{equation}

Define the usual surrogate objective on the augmented MDP:
\[
L_h(\pi_{k+1};\pi_k)
:=
C_h^{\pi_{k}}(\bar \mu_0)
+
\sum_{t=0}^{H-1}\E_{\bar{s}\sim \bar{d}^\pi_t,\;a\sim\pi_{k+1}(\cdot\mid s)}
\big[A_h^{\pi_k}(\bar{s},a)\big].
\]
Then \eqref{eq:one_step_aug_bound} becomes
\begin{equation}
\label{eq:Jh_le_surrogate_plus_TR}
C_h^{\pi_{k+1}}(\bar \mu_0)
\;\le\;
L_h(\pi_{k+1};\pi_k)
+
2H\,\epsilon_h^{\pi_{k+1}}\cdot
\sum_{t=0}^{H-1}\E_{\bar{s}\sim \bar{d}^\pi_k}\Big[D_{\mathrm{TV}}\big(\pi_{k+1}(\cdot\mid s),\pi_k(\cdot\mid s)\big)\Big].
\end{equation}

By \eqref{eq:stein-tr}, our update enforces a marginalized surrogate safety constraint
\[
L_h(\pi_{k+1};\pi_k)\le \varepsilon - m_{k}.
\]
Plugging this into \eqref{eq:Jh_le_surrogate_plus_TR} yields the one-step bound
\begin{equation}
\label{eq:stein_gate_one_step}
C_h^{\pi_{k+1}}(\bar \mu_0)
\;\le\;
d - m_{k}
+
2H\,\epsilon_h^{\pi_{k+1}}\cdot
\sum_{t=0}^{H-1}\E_{\bar{s}\sim \bar{d}^\pi_k}\Big[D_{\mathrm{TV}}\big(\pi_{k+1}(\cdot\mid s),\pi_k(\cdot\mid s)\big)\Big].
\end{equation}

Therefore, if the Stein margin satisfies 
\[
\lambda_k\ge \frac{2H\,\epsilon_h^{\pi_{k+1}}\cdot
\sum_{t=0}^{H-1}\E_{\bar{s}\sim \bar{d}^\pi_k}\Big[D_{\mathrm{TV}}\big(\pi_{k+1}(\cdot\mid s),\pi_k(\cdot\mid s)\big)\Big]}{\underline{SD}_{\rm tot}^+(q_{\pi_k}\|p_{\mathrm{ref}})}
\]
then the following relation holds with probability at least $1-2\delta$,
\[
C_h^{\pi_{k+1}}(\bar \mu_0)\;\le\;
\varepsilon.
\]

\section{Practical SteinGate}
\label{sec:practical}
In subsection \ref{subsec:hybrid_ref}, we provided a decomposition of the tail-risk functional into three components: (i) the reference risk under the hybrid safe distribution, (ii) a boundary mass discrepancy term, and (iii) an interior Stein discrepancy term weighted by the interior mass. This structure naturally suggests a plug-in estimator in which each quantity appearing on ~\eqref{eq:hybrid-certif} is approximated from finite rollout samples.

First, the reference risk $\mathbb{E}_{p_{\mathrm{ref}}}[h(C)]$ depends only on the chosen safe reference and can therefore be computed offline or by Monte Carlo sampling from the interior reference density $p_{\mathrm{int}}$. This term establishes a baseline level of admissible tail risk.

Second, the discrete discrepancy $SD_{\mathrm{disc}}(q \,\|\, p_{\mathrm{ref}})$ depends only on the boundary masses $a_0(q)$ and $a_1(q)$. In the bounded-cost setting, these quantities correspond to the probability that trajectories accumulate negligible cost or saturate at the maximal violation level. Estimating these probabilities reduces to counting the fraction of rollout samples falling within small neighborhoods of the endpoints. Importantly, this step preserves the asymmetry of the theoretical bound: only loss of safe mass and excess violation mass contribute to the certificate, while harmless shifts in the opposite direction are ignored.

Third, the interior discrepancy $SD_{\mathrm{int}}(q_{\mathrm{int}} \,\|\, p_{\mathrm{int}})$ is approximated using a kernelized Stein discrepancy relative to the interior reference density. This requires only a Stein operator for $p_{\mathrm{int}}$ and a function class over the interior. Restricting the discrepancy computation to the interior samples implements the conditional distribution $q_{\mathrm{int}}$, while the empirical interior mass $w(q)$ provides the corresponding weight. 

Combining these three empirical components yields a certificate of equation \ref{eq:hybrid-certif},
which is a direct numerical instantiation of the theoretical upper bound in Lemma~\ref{lem:interior_stein_domination}. The only approximation introduced is the replacement of population expectations and discrepancies by their sample analogues. As the number of rollout samples increases and the boundary neighborhoods shrink, these estimates converge to their population counterparts and the empirical certificate approaches the theoretical bound, up to the vanishing boundary approximation term $\varepsilon_{\mathrm{bdry}}$.

This construction ensures that the practical estimator remains structurally aligned with the theory: boundary mass shifts and interior distributional mismatch are controlled by separate terms with distinct gains, and both contribute additively to the final certificate. No additional modeling assumptions are introduced beyond those required by the hybrid Stein bound itself. 

Finally, we incorporate a lightweight empirical guard that replaces the Stein bound with the empirical risk estimate in regimes where the boundary discrepancy vanishes and the empirical risk is already below the reference level. This step does not alter the form of the certificate in the unsafe regime; it merely reduces variance when the distribution is clearly safe and the bound is nearly tight. The resulting procedure preserves the conservativeness of the Stein certificate while improving numerical stability in benign regimes. The implementation can be found at
\href{https://github.com/yassineCh/SteinGate}{https://github.com/yassineCh/SteinGate}.

\begin{algorithm}[h]
\caption{Hybrid Stein Certificate Estimator}
\label{alg:stein_estimator}
\begin{algorithmic}[1]
\REQUIRE Costs $\mathcal{C}$; Limit $L$; Reference params $\alpha, \beta, a^*$; Weights $\lambda_{\text{disc}}, \lambda_{\text{stein}}$; Risk budget $\varepsilon$
\ENSURE Certificate Bound $U$, Decision $d \in \{\textsc{Safe}, \textsc{Unsafe}\}$

\STATE \textbf{1. Normalization}
\STATE Map costs to $x \in [0, 1]$: $x_i \leftarrow \text{clamp}(c_i/L, 0, 1)$
\STATE Partition into sets: $X_0$ (zeros), $X_1$ (ones), $X_{\text{int}}$ (interior)
\STATE Compute masses: $\hat{a}_0 \leftarrow |X_0|/n, \quad \hat{a}_1 \leftarrow |X_1|/n$

\STATE \textbf{2. Reference Risk}
\STATE $h_{\text{ref}} \leftarrow \mathbb{E}_{z \sim \text{Beta}(\alpha, \beta)} [\sigma(z)]$ \hfill \COMMENT{Precomputed baseline}

\STATE \textbf{3. Compute Stein Penalties} 
\STATE $\mathcal{D}_{\text{disc}} \leftarrow \max(0, \hat{a}_1 - a_1^*) + \max(0, a_0^* - \hat{a}_0)$ \hfill \COMMENT{Mass mismatch}
\IF{$|X_{\text{int}}| > 1$}
    \STATE Compute KSD $\mathcal{D}_{\text{ksd}}$ on $X_{\text{int}}$ using score $\nabla \log \text{Beta}(\cdot; \alpha, \beta)$
\ELSE
    \STATE $\mathcal{D}_{\text{ksd}} \leftarrow 0$
\ENDIF

\STATE \textbf{4. Upper Bound Construction}
\STATE $U_{\text{stein}} \leftarrow h_{\text{ref}} + \lambda_{\text{disc}} \mathcal{D}_{\text{disc}} + \lambda_{\text{stein}} \mathcal{D}_{\text{ksd}}$

\STATE \textbf{5. Empirical Guard (Heuristic Override)}
\STATE $\hat{h}_{\text{emp}} \leftarrow \frac{1}{n} \sum \sigma(x_i)$
\IF{$\hat{a}_1 \approx 0$ \AND $\hat{h}_{\text{emp}} \le h_{\text{ref}}$}
    \STATE $U \leftarrow \hat{h}_{\text{emp}}$ \hfill \COMMENT{Trust empirical if trivially safe}
\ELSE
    \STATE $U \leftarrow U_{\text{stein}}$ \hfill \COMMENT{Trust Stein bound otherwise}
\ENDIF

\end{algorithmic}
\end{algorithm}

\section{Additional Results}
\label{sec:additional_results}
\subsection{Additional Benchmarks}
\begin{figure}[!htbp]
    \centering
    \includegraphics[width=\linewidth]{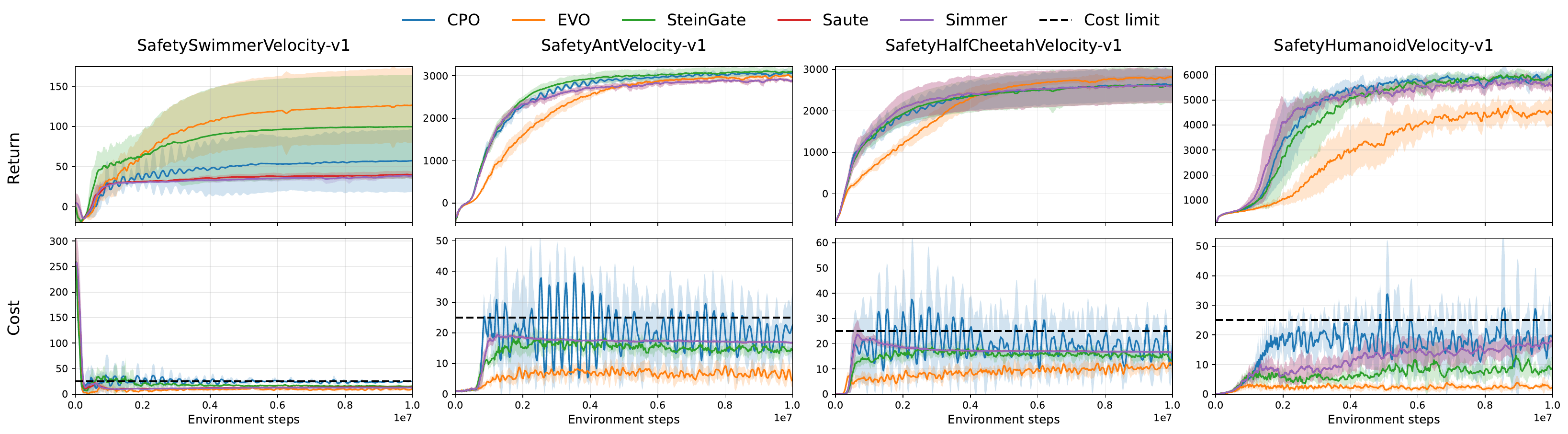}
    \caption{Performance on velocity benchmarks. The top row shows episodic return and the bottom row shows episodic cost during training. The dashed lines indicate the cost limit.}
    \label{fig:add_benchmarks}
\end{figure}

Figure \ref{fig:add_benchmarks} extends our evaluation to the Safety MuJoCo suite, which tasks agents with locomotion control under velocity-based constraints. Consistent with the navigation results, SteinGate demonstrates superior constraint adherence, rapidly converging to feasible policies while maintaining competitive forward velocities. In contrast, \textsc{CPO} suffers from significant instability, characterized by large cost oscillations and frequent safety violations, particularly in the complex dynamics of \texttt{Ant} and \texttt{HalfCheetah}. While \textsc{EVO} succeeds in maintaining safety, it again exhibits excessive conservatism, resulting in slower convergence and lower asymptotic returns, most notably on \texttt{Humanoid}. Although state-augmentation methods (\textsc{Saute}, \textsc{Simmer}) improve upon CPO's safety profile, they frequently operate precariously close to the boundary with intermittent violations. These results confirm SteinGate's robustness, demonstrating that its safety mechanism generalizes effectively from spatial navigation to velocity-constrained locomotion without sacrificing task performance.

\subsection{Additional Baselines}
\begin{figure}
    \centering
    \includegraphics[width=\linewidth]{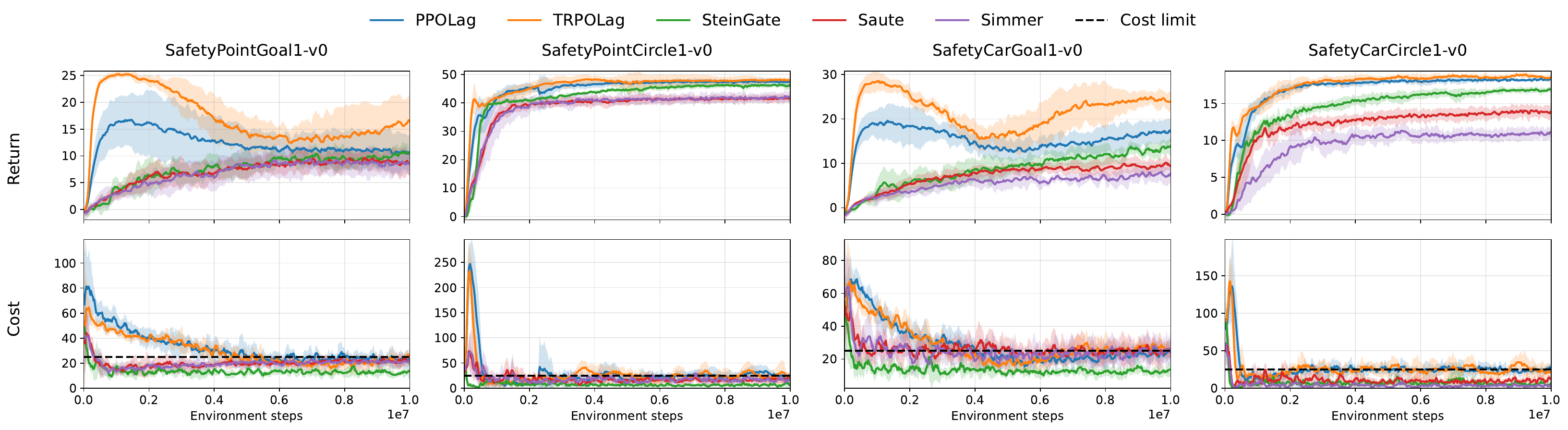}
    \caption{Comparison of SteinGate with additional baselines. The top row shows episodic return, and the bottom row shows episodic cost during training. The dashed lines indicate the cost limit.}
    \label{fig:add_baselines}
\end{figure}

Figure \ref{fig:add_baselines} expands our comparative analysis to include Lagrangian methods (PPO-Lag, TRPO-Lag) \citep{schulman2017proximal, schulman2015trust} and PPO-based implementations of state augmentation techniques (Saute-PPO, Simmer-PPO) \citep{sootla2022saute, sootla2022effects}. While the main text utilizes TRPO-based baselines to align with the backbones of CPO, EVO, and SteinGate, these additional comparisons isolate the impact of the underlying optimizer from the safety mechanism itself. The Lagrangian baselines prioritize reward maximization but fail to strictly enforce constraints, resulting in sustained operation above the cost limit. Although switching the backbone of Saute and Simmer to PPO reduces the magnitude of these violations compared to Lagrangian approaches, they still exhibit frequent excursions beyond the safety boundary. In contrast, SteinGate demonstrates superior constraint adherence without requiring the hyperparameter sensitivity often associated with Lagrangian multipliers, maintaining cost trajectories that tightly track the prescribed limit while delivering competitive task performance.

\begin{figure}[h]
    \centering
    \includegraphics[width=\linewidth]{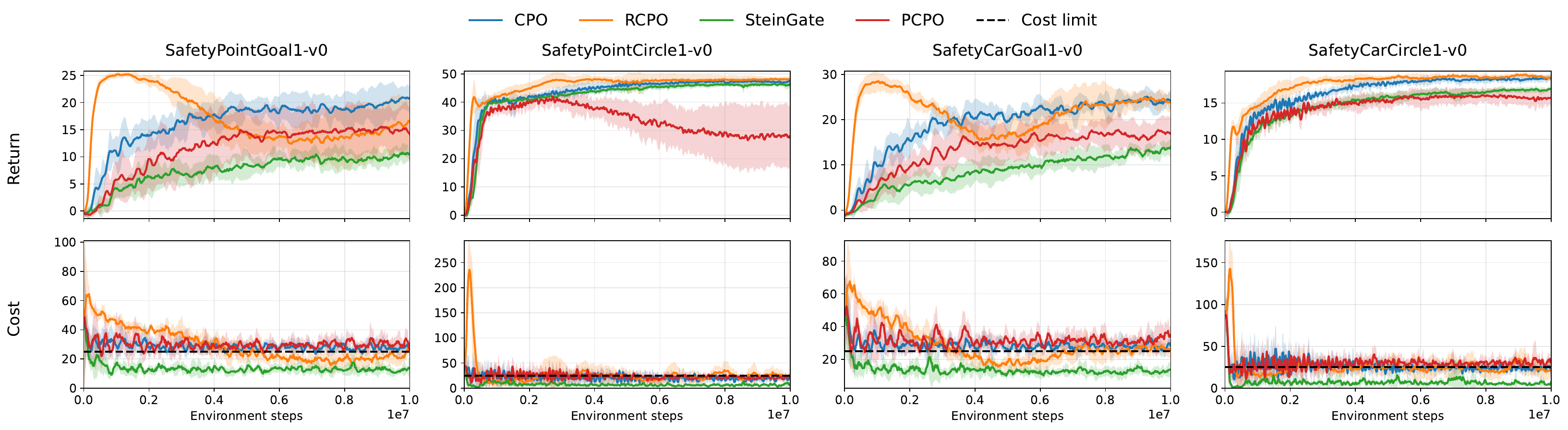}
    \caption{Comparison of SteinGate with CPO variants. The top row shows episodic return and the bottom row shows episodic cost during training. The dashed lines indicates the cost limit.}
    \label{fig:cpo_variants}
\end{figure}

\noindent \textbf{Comparison with CPO Variants.}
Figure \ref{fig:cpo_variants} compares SteinGate with standard CPO and two of its variants, \textsc{RCPO} \citep{tessler2018reward} and PCPO \citep{yang2020projection}, across four Safety Gymnasium tasks. Although all methods share the same trust-region constrained optimization backbone, their behaviors differ substantially. Standard CPO and \textsc{RCPO} achieve relatively high returns but exhibit sustained cost levels close to or above the constraint, indicating frequent violations during training. PCPO also fails to consistently satisfy the constraint and shows noticeable excursions above the cost limit across environments. In contrast, SteinGate maintains systematically lower episodic cost while achieving competitive return. These results indicate that the improvement provided by SteinGate is not simply due to the choice of backbone optimizer, but arises from the proposed Stein-based gating mechanism, which offers more reliable control of constraint satisfaction than alternative variants of CPO.

\section{Experimental Details}
\subsection{Sample Size and Training Dynamics}
\label{ap:sample_size}

Our experimental framework is grounded in the \texttt{OmniSafe} library \citep{ji2024omnisafe}. We implement SteinGate by extending the \texttt{OmniSafe} version of CPO. Similarly, we utilize standard \texttt{OmniSafe} implementations for baselines such as Saute and Simmer. For EVO, we adopt the official standalone implementation; however, we note that it is also constructed upon the omnisafe CPO. To ensure a strictly fair comparison, we maintained the default architectural hyperparameters and interaction budgets used for the baselines. In on-policy algorithms, the number of cost samples available for a safety update is determined by the \texttt{steps\_per\_epoch} (interaction budget per update) and the environment’s \texttt{episode\_length}. To accelerate training on our hardware, we used \texttt{vector\_env\_nums=8} (eight parallel environments). Given the task parameters, this resulted in an effective sample size of $N = 16$ cost episodes per epoch for Goal tasks (episode length 1000) and $N = 40$ for Circle tasks (episode length 500).

A core advantage of SteinGate is its ability to operate without the extensive replay buffers or importance resampling required by EVT-based methods such as EVO. To verify that the safety guarantees of SteinGate hold independently of the specific OmniSafe configuration, we conducted an ablation on \texttt{CarGoal1} by modulating the \texttt{steps\_per\_epoch} parameters to obtain batch sizes of $N \in \{8, 16, 32\}$.

It is important to note that varying $N$ in this framework implicitly changes the number of policy updates per total environment steps (e.g., larger batches imply fewer total updates for a fixed interaction budget). Despite this confounding factor, the results in paragraph~\ref{pg:sample_size_ablation} show SteinGate effectiveness. This ablation confirms that SteinGate is not brittle with respect to sample size and can be integrated into standard on-policy implementations (e.g., CPO) without requiring the collection of off-policy data to stabilize the safety estimator.

\subsection{Environment Descriptions}
\label{ap:envs_desc}

We evaluate SteinGate on a collection of continuous-control environments drawn from the \texttt{Safety-Gymnasium} benchmark suite, which is built on the MuJoCo simulator \citep{ray2019benchmarking, todorov2012mujoco}. These environments are designed to study reinforcement learning under explicit safety constraints by associating penalties with hazardous states or behaviors.

\noindent \textbf{Point and Car Tasks.}
The Point and Car agents are evaluated on both \textsc{Goal} and \textsc{Circle} variants. In the \textsc{Goal} tasks, the agent must navigate toward a sequence of target locations while avoiding hazardous regions in the environment. After each successful reach, a new goal position is sampled, requiring continual navigation under safety constraints. In the \textsc{Circle} tasks, the agent is encouraged to follow a predefined circular trajectory. Deviations outside the designated safe region incur costs, creating a trade-off between tracking accuracy and constraint satisfaction. The Car variants introduce additional control complexity through nonholonomic dynamics, making safe navigation more challenging than for the Point agent.

\noindent \textbf{Velocity-Constrained Locomotion Tasks.}
We additionally consider velocity-limited locomotion tasks for articulated agents, including \textsc{Ant}, \textsc{HalfCheetah}, \textsc{Swimmer}, and \textsc{Humanoid}. In these environments, the agent is rewarded for forward progress while being penalized for exceeding a specified velocity threshold. The safety constraint therefore requires the agent to coordinate its movement to achieve steady locomotion without violating speed limits. These tasks capture a different class of safety considerations, where constraints apply to the dynamics of motion rather than spatial regions.


\subsection{Hyperparameters}
\label{ap:hyperparameters}

All experiments were conducted using the default hyperparameter configurations provided by the \texttt{OmniSafe} library for the corresponding baseline algorithms. We did not perform task-specific tuning or additional hyperparameter searches. This choice ensures a fair and reproducible comparison and isolates the effect of the proposed SteinGate mechanism from confounding optimization or architectural choices. Unless otherwise specified, network architectures, learning rates, batch sizes, and interaction budgets follow the standard \texttt{OmniSafe} settings.

For SteinGate, the practical implementation introduces a number of parameters related to the construction of the Stein certificate. These parameters are shared across all tasks and are reported in Table~\ref{tab:steingate_hyperparams}.

For EVO, the official implementation does not fully specify several hyperparameters related to buffer usage and percentile thresholds. In particular, the following parameters were not documented in the released code: \texttt{use\_cost\_buffer}, \texttt{use\_reward\_buffer}, \texttt{cost\_buffer\_size}, \texttt{reward\_buffer\_size}, \texttt{initial\_distribution\_percentile}, \texttt{initial\_distribution\_reward\_percentile}, \texttt{cost\_buffer\_percentile}, \texttt{reward\_buffer\_percentile}, and \texttt{oversample\_factor}. We initialized a subset of these parameters based on the descriptions provided in the EVO paper and performed limited sweeps over the remaining unspecified values. From these runs, we selected the configuration that achieved the best trade-off between reward performance and constraint satisfaction. The final values used in our experiments are summarized in Table~\ref{tab:evo_hyperparams}.

\begin{table}[H]
\centering
\caption{Hyperparameters of the Hybrid Stein Certificate Estimator}
\label{tab:steingate_hyperparams}
\begin{tabular}{llll}
\toprule
\textbf{Category} & \textbf{Hyperparameter} & \textbf{Symbol} & \textbf{Default Value} \\
\midrule

\multicolumn{4}{l}{\textit{Reference Distribution}} \\
 & $\texttt{alpha}$ & $\alpha$ & 2.0 \\
 & $\texttt{beta}$ & $\beta$ & 5.0 \\

\multicolumn{4}{l}{\textit{Risk Mapping $\sigma(\cdot)$}} \\
 & $\texttt{u\_norm}$ & -- & 0.5 \\
 & $\texttt{eta}$ & -- & 0.02 \\

\multicolumn{4}{l}{\textit{Risk Budget}} \\
 & $\texttt{eps\_target}$ & $\varepsilon$ & 0.10 \\

\multicolumn{4}{l}{\textit{Stein Penalty Weights}} \\
 & $\texttt{discrete\_weight}$ & $\lambda_{\text{disc}}$ & 1.0 \\
 & $\texttt{stein\_factor}$ & $\lambda_{\text{stein}}$ & 0.10 \\

\multicolumn{4}{l}{\textit{KSD Stability}} \\
 & $\texttt{min\_bandwidth}$ & -- & 0.05 \\
 & $\texttt{interior\_clip}$ & -- & $1\times10^{-5}$ \\

\bottomrule
\end{tabular}
\end{table}

\begin{table}
\centering
\small
\caption{EVO hyperparameters not specified in the official implementation and selected in this work.}
\label{tab:evo_hyperparams}
\begin{tabular}{l c}
\toprule
\textbf{Hyperparameter} & \textbf{Value} \\
\midrule
\texttt{use\_cost\_buffer} & True \\
\texttt{use\_reward\_buffer} & True \\
\texttt{cost\_buffer\_size} & 50 \\
\texttt{reward\_buffer\_size} & 100 \\
\texttt{initial\_distribution\_percentile} & 0.5 \\
\texttt{initial\_distribution\_reward\_percentile} & 0.5 \\
\texttt{cost\_buffer\_percentile} & 95 \\
\texttt{reward\_buffer\_percentile} & 90 \\
\texttt{oversample\_factor} & 4 \\
\bottomrule
\end{tabular}
\end{table}


\end{document}